\documentclass[twoside]{article}
\usepackage{aistats2016}
\usepackage{amsmath}
\usepackage{amsfonts}
\newcommand{\mb}{\mathbf}
\usepackage{amsthm}

\newtheorem{lemma}{Lemma}
\newtheorem{thm}{Theorem}

\usepackage{hyperref}
\usepackage{graphicx}

\newcommand\ci{\perp\!\!\!\perp}
\begin{document}

%

%

\twocolumn[

\aistatstitle{On the Generalization Error Bounds of Neural Networks under Diversity-Inducing Mutual Angular Regularization}

\aistatsauthor{ Pengtao Xie \And Yuntian Deng \And Eric Xing }

\aistatsaddress{ Carnegie Mellon University \And Carnegie Mellon University \And Carnegie Mellon University } 
]

\begin{abstract}

Recently diversity-inducing regularization methods for latent variable models (LVMs), which encourage the components in LVMs to be diverse, have been studied to address several issues involved in latent variable modeling: (1) how to capture long-tail patterns underlying data; (2) how to reduce model complexity without sacrificing expressivity; (3) how to improve the interpretability of learned patterns. While the effectiveness of diversity-inducing regularizers such as the mutual angular regularizer \cite{xie2015diversifying} has been demonstrated empirically, a rigorous theoretical analysis of them is still missing. In this paper, we aim to bridge this gap and analyze how the mutual angular regularizer (MAR) affects the generalization performance of supervised LVMs. We use neural network (NN) as a model instance to carry out the study and the analysis shows that increasing the diversity of hidden units in NN would reduce estimation error and increase approximation error. In addition to theoretical analysis, we also present empirical study which demonstrates that the MAR can greatly improve the performance of NN and the empirical observations are in accordance with the theoretical analysis.  
\end{abstract}

\section{Introduction}
One central task in machine learning (ML) is to extract underlying patterns from observed
data \cite{bishop2006pattern,han2011data,fukunaga2013introduction}, which is essential for making effective use of big data for many applications \cite{council2013frontiers,jordan2015machine}.
Among the various ML models and algorithms designed for pattern discovery, latent variable models (LVMs) \cite{rabiner1989tutorial,bishop1998latent,knott1999latent,blei2003latent,hinton2006fast,
airoldi2009mixed,blei2014build} or latent space models (LSMs) \cite{rumelhart1985learning,deerwester1990indexing,olshausen1997sparse,
lee1999learning,xing2002distance} are a large family of models providing a principled and effective way to uncover knowledge hidden behind data and have been widely used in text mining \cite{deerwester1990indexing,blei2003latent}, computer vision \cite{olshausen1997sparse,lee1999learning}, speech recognition \cite{rabiner1989tutorial,hinton2012deep}, computational biology \cite{xing2007bayesian,song2009keller} and recommender systems \cite{gunawardana2008tied,koren2009matrix}. 

Although LVMs have now been widely used, several new challenges have emerged due to the dramatic growth of volume and complexity of data:
(1) In the event that the popularity of patterns behind big data is distributed in a power-law fashion, where a few dominant patterns occur frequently whereas most patterns in the long-tail region are of low popularity \cite{wang2014peacock,xie2015diversifying}, standard LVMs are inadequate to capture the long-tail patterns, which can incur significant information loss \cite{wang2014peacock,xie2015diversifying}. (2) To cope with the rapidly growing complexity of patterns present in big data, ML practitioners typically increase the size and capacity of LVMs, which incurs great challenges for model training, inference, storage and maintenance \cite{xie2015learning}. How to reduce model complexity without compromising expressivity is a challenging issue. (3) There exist substantial redundancy and overlapping amongst patterns discovered by existing LVMs from massive data, making them hard to interpret \cite{wang2015rubik}.

To address these challenges, several recent works \cite{Zou_priorsfor,xie2015diversifying,xie2015learning} have investigated a diversity-promoting regularization technique for LVMs, which controls the geometry of the latent space during learning to encourage the learned latent components of LVMs to be diverse in the sense that they are favored to be mutually "different" from each other. First, concerning the long-tail phenomenon in extracting latent patterns (e.g., clusters, topics) from data: if the model components are biased to be far apart from each other, then one would expect that such components will tend to be less overlapping and less aggregated over dominant patterns (as one often experiences in standard clustering algorithms \cite{Zou_priorsfor}), and therefore more likely to capture the long-tail patterns.
Second, reducing model complexity without sacrificing expressivity: if the model components are preferred to be different from each other,
then the patterns captured by different components are likely to have less redundancy and hence complementary to each other. Consequently, it is possible to use a small set of components to sufficiently capture a large proportion of patterns. Third, improving the interpretability of the learned components: if model components are encouraged to be distinct from each other and non-overlapping, then it would be cognitively easy for human to associate each component to an object or concept in the physical world. Several diversity-inducing regularizers such as Determinantal Point Process \cite{Zou_priorsfor}, mutual angular regularizer \cite{xie2015diversifying} have been proposed to promote diversity in various latent variable models including Gaussian Mixture Model \cite{Zou_priorsfor}, Latent Dirichlet Allocation \cite{Zou_priorsfor}, Restricted Boltzmann Machine \cite{xie2015diversifying}, Distance Metric Learning \cite{xie2015diversifying}.

While the empirical effectiveness of diversity-inducing regularizers has been demonstrated in \cite{Zou_priorsfor,xie2015diversifying,xie2015learning}, their theoretical behaviors are still unclear. In this paper, we aim to bridge this gap and make the first attempt to formally understand why and how introducing diversity into LVMs can lead to better modeling effects. We focus on the mutual angular regularizer proposed in \cite{xie2015diversifying} and analyze how it affects the generalization performance of supervised latent variable models. Specifically, we choose neural network (NN) as a model instance to carry out the analysis while noting that the analysis could be extended to other LVMs such as Restricted Boltzmann Machine and Distance Metric Learning. The major insights distilled from the analysis are: as the diversity (which will be made precise later) of hidden units in NN increases, the estimation error of NN decreases while the approximation error increases; thereby the overall generalization error (which is the sum of estimation error and generalization error) reaches the minimum if an optimal diversity level is chosen. In addition to the theoretical study, we also conduct experiments to empirically show that with the mutual angular regularization, the performance of neural networks can be greatly improved. And the empirical results are consistent with the theoretical analysis.




The major contributions of this paper include:
\begin{itemize}
\item We propose a diversified neural network with mutual angular regularization (MAR-NN).
\item We analyze the generalization performance of MAR-NN, and show that the mutual angular regularizer can help reduce generalization error. 
\item Empirically, we show that mutual angular regularizer can greatly improve the performance of NNs and the experimental results are in accordance with the theoretical analysis.
\end{itemize}

The rest of the paper is organized as follows. Section 2 introduces mutual angle regularized neural networks (MAR-NNs). The estimation and approximation errors of MAR-NN are analyzed in Section 3. Section 4 presents empirical studies of MAR-NN. Section 5 reviews related works and Section 6 concludes the paper. 

\section{Diversify Neural Networks with Mutual Angular Regularizer}
In this section, we review diversity-regularized latent variable models and propose diversified neural networks with mutual angular regularization.

\subsection{Diversity-Promoting Regularization of Latent Variable Models}
Uncover latent patterns from observed data is a central task in big data analytics \cite{bishop2006pattern,han2011data,council2013frontiers,fukunaga2013introduction,jordan2015machine}. Latent variable models \cite{rumelhart1985learning,rabiner1989tutorial,deerwester1990indexing,
olshausen1997sparse,bishop1998latent,knott1999latent,lee1999learning,
xing2002distance,blei2003latent,hinton2006fast,airoldi2009mixed,blei2014build} elegantly fit into this task. The knowledge and structures hidden behind data are usually composed of multiple \textit{patterns}. For instance, the semantics underlying documents contains a set of \textit{themes} \cite{hofmann1999probabilistic,blei2003latent}, such as politics, economics and education.
Accordingly, latent variable models are parametrized by multiple \textit{components} where each component aims to capture one pattern in the knowledge and is represented with a parameter vector. For instance, the components in Latent Dirichlet Allocation \cite{blei2003latent} are called \textit{topics} and each topic is parametrized by a multinomial vector. 

To address the aforementioned three challenges in latent variable modeling: the skewed distribution of pattern popularity, the conflicts between model complexity and expressivity and the poor interpretability of learned patterns, recent works \cite{Zou_priorsfor,xie2015diversifying,xie2015learning} propose to diversify the components in LVMs, by solving a regularized problem:
\begin{equation}
\textrm{max}_{\mb{A}}\quad \mathcal{L}(\mb{A})+\lambda \Omega(\mb{A})
\end{equation}
where each column of $\mb{A}\in \mathbb{R}^{\mathsf{d\times k}}$ is the parameter vector of a component, $\mathcal{L}(\mb{A})$ is the objective function of the original LVM, $\Omega(\mb{A})$ is a regularizer encouraging the components in $\mb{A}$ to be diverse and $\lambda$ is a tradeoff parameter. Several regularizers have been proposed to induce diversity, such as Determinantal Point Process \cite{Zou_priorsfor}, mutual angular regularizer \cite{xie2015diversifying}.

Here we present a detailed review of the mutual angular regularizer \cite{xie2015diversifying} as our theoretical analysis is based on it. This regularizer is defined with the rationale that if each pair of components are mutually different, then the set of components are diverse in general. They utilize the non-obtuse angle $\theta_{ij}=\textrm{arccos}(\frac{|\mb{a}_i\cdot\mb{a}_j|}{\|\mb{a}\|_i\|\mb{a}\|_j})$ to measure the dissimilarity between component $\mb{a}_i$ and $\mb{a}_j$ as angle is insensitive to geometry transformations of vectors such as scaling, translation, rotation, etc. Given a set of components, angles $\{\theta_{ij}\}$ between each pair of components are computed and the MAR is defined as the mean of these angles minus their variance
\begin{equation}
\begin{array}{l}
\Omega(\mb{A})=\frac{1}{K(K-1)}\sum_{i=1}^{K}\sum_{j=1,j\neq i}^{K}\theta_{ij}
-\gamma\frac{1}{K(K-1)}\\\sum_{i=1}^{K}\sum_{j=1,j\neq i}^{K}(\theta_{ij}-\frac{1}{K(K-1)}\sum_{p=1}^{K}\sum_{q=1,q\neq p}^{K}\theta_{pq})^{2}
\end{array}
\end{equation}
where $\gamma>0$ is a tradeoff parameter between mean and variance.
The mean term summarizes how these vectors are different from each on the whole. A larger mean indicates these vectors share a larger angle in general, hence are more diverse. The variance term is utilized to encourage the vectors to evenly spread out to different directions. A smaller variance indicates that these vectors are uniformly different from each other.

\subsection{Neural Network with Mutual Angular Regularization}

Recently, neural networks (NNs) have shown great success in many applications, such as speech recognition \cite{hinton2012deep}, image classification \cite{krizhevsky2012imagenet}, machine translation \cite{bahdanau2014neural}, etc. Neural networks are nonlinear models with large capacity and rich expressiveness. If trained properly, they can capture the complex patterns underlying data and achieve notable performance in many machine learning tasks. NNs are composed of multiple layers of computing units and units in adjacent layers are connected with weighted edges. NNs are a typical type of LVMs where each hidden unit is a component aiming to capture the latent features underlying data and is characterized by a vector of weights connecting to units in the lower layer.

We instantiate the general framework of diversity-regularized LVM to neural network and utilize the mutual angular regularizer to encourage the hidden units (precisely their weight vectors) to be different from each other, which could lead to several benefits: (1) better capturing of long-tail latent features; (2) reducing the size of NN without compromising modeling power. 
Let $\mathcal{L}(\{\mb{A}_{i}\}_{i=0}^{l-1})$ be the loss function of a neural network with $l$ layers where $\mb{A}_{i}$ are the weights between layer $i$ and layer $i+1$, and each column of $\mb{A}_{i}$ corresponds to a unit. A diversified NN with mutual angular regularization (MAR-NN) can be defined as
\begin{equation}
\begin{array}{ll}
\textrm{min}_{\{\mb{A}_{i}\}_{i=0}^{l-1}}&\mathcal{L}(\{\mb{A}_{i}\}_{i=0}^{l-1})-\lambda \sum_{i=0}^{l-2}\Omega(\mb{A}_{i})
\end{array}
\end{equation}
where $\Omega(\mb{A}_{i})$ is the mutual angular regularizer and $\lambda>0$ is a tradeoff parameter. Note that the regularizer is not applied to $\mb{A}_{l-1}$ since in the last layer are output units which are not latent components.

\section{Generalization Error Analysis}
In this section, we analyze how the mutual angular regularizer affects the generalization error of neural networks. Let $L(f)=\mathbb{E}_{(\mb{x},y)\sim p^*}[\ell(\mb{x},y,f)]$ denote the generalization error of hypothesis $f$, where $p^*$ is the distribution of input-output pair $(\mb{x},y)$ and $\ell(\cdot)$ is the loss function. Let $f^*\in \textrm{argmin}_{f\in \mathcal{F}}L(f)$ be the expected risk minimizer. Let $\hat{L}(f)=\frac{1}{n}\sum_{i=1}^{n}\ell(\mb{x}^{(i)},y^{(i)},f)$ be the training error and $\hat{f}\in \textrm{argmin}_{f\in \mathcal{F}}\hat{L}(f)$ be the empirical risk minimizer. We are interested in the generalization error $L(\hat{f})$ of the empirical risk minimizer $\hat{f}$, which can be decomposed into two parts $L(\hat{f})=L(\hat{f})-L(f^*)+L(f^*)$, where $L(\hat{f})-L(f^*)$ is the estimation error (or excess risk) and $L(f^*)$ is the approximation error. The estimation error represents how well the algorithm is able to learn and usually depends on the complexity of the hypothesis and the number of training samples. A lower hypothesis complexity and a larger amount of training data incur lower estimation error bound. The approximation error indicates how expressive the hypothesis set is to effectively approximate the target function.

Our analysis below shows that the mutual angular regularizer can reduce the generalization error of neural networks. We assume with high probability $\tau$, the angle between each pair of hidden units is lower bounded by $\theta$. $\theta$ is a formal characterization of diversity. The larger $\theta$ is, the more diverse these hidden units are. The analysis in the following sections suggests that $\theta$ incurs a tradeoff between estimation error and approximation error: the larger $\theta$ is, the smaller the estimation error bound is and the larger the approximation error bound is. Since the generalization error is the sum of estimation error and approximation error, $\theta$ has an optimal value to yield the minimal generalization error. In addition, we can show that under the same probability $\tau$, increasing the mutual angular regularizer can increase $\theta$. Given a set of hidden units $\mb{A}$ learned by the MAR-NN, we assume their pairwise angles $\{\theta_{ij}\}$ are $i.i.d$ samples drawn from a distribution $p(X)$ where the expectation and variance of random variable $X$ is $\mu$ and $\sigma$ respectively. Lemma \ref{lem:mar_bd} states that $\theta$ is an increasing function of $\mu$ and decreasing function of $\sigma$. By the definition of MAR, it encourages larger mean and smaller variance. Thereby, the larger the MAR is, the larger $\theta$ is. Hence properly controlling the MAR can generate a desired $\theta$ that produces the lowest generalization error.

\begin{lemma}
\label{lem:mar_bd} With probability at least $\tau$, we have 
$X\geq \theta=\mu-\sqrt{\frac{\sigma}{1-\tau}}$
\begin{proof}
According to Chebyshev inequality \cite{wasserman2013all}, 
\begin{equation}
\frac{\sigma}{t^2}\geq p(|X-\mu|>t)\geq p(X<\mu-t)
\end{equation}
Let $\theta=\mu-t$, then $p(X<\theta)\leq \frac{\sigma}{(\mu-\theta)^2}$. Hence $p(X\geq\theta)\geq 1-\frac{\sigma}{(\mu-\theta)^2}$. Let $\tau=1-\frac{\sigma}{(\mu-\theta)^2}$, then $\theta=\mu-\sqrt{\frac{\sigma}{1-\tau}}$.
\end{proof}
\end{lemma}

\subsection{Setup}
\label{sec:setup}
For the ease of presentation, we first consider a simple neural network whose setup is described below. Later on we extend the analysis to more complicated neural networks.
\begin{itemize}
\item Network structure: one input layer, one hidden layer and one output layer
\item Activation function: Lipschitz continuous function $h(t)$ with constant $L$. Example: rectified linear $h(t)=\textrm{max}(0,t)$, $L=1$; tanh $h(t)=\textrm{tanh}(t)$, $L=1$; sigmoid $h(t)=\textrm{sigmoid}(t)$, $L=0.25$.
\item Task: univariate regression
\item Let $\mb{x}\in \mathbb{R}^d$ be the input vector with $\|\mb{x}\|_2\leq C_1$
\item Let $y$ be the response value with $|y|\leq C_2$
\item Let $\mb{w_j}\in \mathbb{R}^d$ be the weights connecting to the $j$-th hidden unit, $j=1,\cdots,m$, with $\|\mb{w_j}\|_2\leq C_3$. Further, we assume with high probability $\tau$, the angle $\rho(\mb{w_i},\mb{w_j})=\textrm{arccos}(\frac{|\mb{w_i}\cdot\mb{w_j}|}{\|\mb{w_i}\|_2\|\mb{w_j}\|_2})$ between $\mb{w_i}$ and $\mb{w_j}$ is lower bounded by a constant $\theta$ for all $i\neq j$.
\item Let $\alpha_j$ be the weight connecting the hidden unit $j$ to the output with $\|\mb{\alpha}\|_2\leq C_4$
\item Hypothesis set: $\mathcal{F}=\{f|f(\mb{x})=\sum\limits_{j=1}^{m}\alpha_j h(\mb{w_j}^{\mathsf{T}}\mb{x})\}$
\item Loss function set: $\mathcal{A}=\{\ell|\ell(\mb{x},y)=(f(\mb{x})-y)^2\}$
\end{itemize}

\subsection{Estimation Error}
We first analyze the estimation error bound of MAR-NN and are interested in how the upper bound is related with the diversity (measured by $\theta$) of the hidden units. The major result is presented in Theorem \ref{thm:est_err}. 
\begin{thm}
\label{thm:est_err}
With probability at least $(1-\delta)\tau$
\begin{equation}
\label{eq:est_err}
\begin{array}{lll}
&&L(\hat{f})-L(f^*)\\
&\leq& 8(\sqrt{\mathcal{J}}+C_2)(2LC_1C_3C_4+C_4|h(0)|)\frac{\sqrt{m}}{\sqrt{n}}\\
&&+ (\sqrt{\mathcal{J}}+C_2)^2\sqrt{\frac{2\log(2/\delta)}{n}}
\end{array}
\end{equation}
where $\mathcal{J}=mC_4^2h^2(0)+L^2C_1^2C_3^2C_4^2((m-1)\cos\theta+1) +\\ 2\sqrt{m}C_1C_3C_4^2L|h(0)|\sqrt{(m-1)\cos\theta+1}$.
\end{thm}

Note that the right hand side is a decreasing function w.r.t $\theta$. A larger $\theta$ (denoting the hidden units are more diverse) would induce a lower estimation error bound. 

\subsubsection{Proof}
\label{sec:thm1_proof}
A well established result in learning theory is that the estimation error can be upper bounded by the Rademacher complexity. We start from the Rademacher complexity, seek a further upper bound of it and show how the diversity of the hidden units affects this upper bound.
The Rademacher complexity $\mathcal{R}_n(\mathcal{A})$ of the loss function set $\mathcal{A}$ is defined as 
\begin{equation}
\begin{array}{lll}
\mathcal{R}_n(\mathcal{A})&=&\mathbb{E}[\textrm{sup}_{\ell\in \mathcal{A}}\frac{1}{n}\sum_{i=1}^{n}\sigma_i \ell(f(\mb{x}^{(i)}), y^{(i)})]\\
\end{array}
\end{equation}
where $\sigma_i$ is uniform over $\{-1,1\}$ and $\{(\mb{x}^{(i)}, y^{(i)})\}_{i=1}^{n}$ are i.i.d samples drawn from $p^*$. The Rademacher complexity can be utilized to upper bound the estimation error, as shown in Lemma \ref{lem:rc_bd}.
\begin{lemma}\cite{anthony1999neural,bartlett2003rademacher,liang2015lecture}
\label{lem:rc_bd}
With probability at least $1-\delta$
\begin{equation}
L(\hat{f})-L(f^*)\leq 4\mathcal{R}_n(\mathcal{A})+B\sqrt{\frac{2\log(2/\delta)}{n}}
\end{equation}
for $B \ge  \sup_{\mb{x}, y, f}|\ell(f(\mb{x}), y)|$
\end{lemma}
Our analysis starts from this lemma and we seek further upper bound of $\mathcal{R}_n(\mathcal{A})$. The analysis needs an upper bound of the Rademacher complexity of the hypothesis set $\mathcal{F}$, which is given in Lemma \ref{lem:rc_f}.
\begin{lemma}
\label{lem:rc_f}
Let $\mathcal{R}_n(\mathcal{F})$ denote the Rademacher complexity of the hypothesis set $\mathcal{F}=\{f|f(\mb{x})=\sum\limits_{j=1}^{m}\alpha_j h(\mb{w_j}^{\mathsf{T}}\mb{x})\}$, then
\begin{equation}
\mathcal{R}_n(\mathcal{F})\leq \frac{2 L C_{1}C_{3}C_{4}\sqrt{m}}{\sqrt{n}} + \frac{C_{4}|h(0)|\sqrt{m}}{\sqrt{n}}
\end{equation}
\end{lemma}
\begin{proof}
\begin{equation}
\begin{array}{lll}
\mathcal{R}_n(\mathcal{F}) &= \mathbb{E}[\sup_{f\in \mathcal{F}}\frac{1}{n}\sum_{i=1}^n\sigma_i\sum_{j=1}^m \alpha_j h(\mb{w_j}^T \mb{x_i})]\\
&= \mathbb{E}[\sup_{f\in \mathcal{F}}\frac{1}{n}\sum_{j=1}^m \alpha_j \sum_{i=1}^n \sigma_i h(\mb{w_j}^T \mb{x_i})]
\end{array}
\end{equation}
Let $\mb{\alpha} = [\alpha_1, \cdots, \alpha_m]^T$ and $\mb{h} = [\sum_{i=1}^n\sigma_ih(\mb{w_1}^T\mb{x_i}),\cdots,\sum_{i=1}^n\sigma_ih(\mb{w_m}^T\mb{x_i})]^T$, the inner product $\mb{\alpha}\cdot\mb{h}\le\|\mb{\alpha}\|_1\|\mb{x}\|_\infty$ as $\|\cdot\|_1$ and $\|\cdot\|_\infty$ are dual norms. Therefore
\begin{equation}
\begin{array}{lll}
&\mb{\alpha}\cdot\mb{h} \\
\le&\|\mb{\alpha}\|_1 \|\mb{h}\|_\infty\\
=& (\sum_{j=1}^m |\alpha_j|)(\max_{j=1,\cdots,m}|\sum_{i=1}^n \sigma_ih(\mb{w_j}^T \mb{x_i})|)\\
\le& \sqrt{m} \|\mb{\alpha}\|_2\cdot \max_{j=1,\cdots,m}|\sum_{i=1}^n\sigma_ih(\mb{w_j}^T\mb{x_i})|\\
\le&\sqrt{m}C_{4}\cdot \max_{j=1,\cdots,m}|\sum_{i=1}^n\sigma_ih(\mb{w_j}^T\mb{x_i})|
\end{array}
\end{equation}
So $\mathcal{R}_n(\mathcal{F})\le \sqrt{m}C_4\mathbb{E}[\sup_{f\in \mathcal{F}}\frac{1}{n}|\sum_{i=1}^n\sigma_ih(\mb{w_j}^T\mb{x_i})|]$. Denote $\mathcal{R}_{||}(\mathcal{F})=\mathbb{E}[\sup_{f\in\mathcal{F}}|\frac{2}{n}\sigma_i f(\mb{x_i})|]$, which is another form of Rademacher complexity used in some literature such as \cite{bartlett2003rademacher}. Let $\mathcal{F}'=\{f'|f'(\mb{x})=h(\mb{w}^T\mb{x})\}$ where $\mb{w}, \mb{x}$ satisfy the conditions specified in Section \ref{sec:setup}, then $\mathcal{R}_n(\mathcal{F})\le\frac{\sqrt{m}C_{4}}{2}\mathcal{R}_{||}(\mathcal{F'})$. 

Let $\mathcal{G} = \{g|g(\mb{x})=\mb{w}^T\mb{x}\}$ where $\mb{w}, \mb{x}$ satisfy the conditions specified in Section \ref{sec:setup}, then $\mathcal{R}_{||}(\mathcal{F'})=\mathcal{R}_{||}(h\circ g)$. Let $h'(\cdot) = h(\cdot) - h(0)$, then $h'(0)=0$ and $h'$ is also L-Lipschitz. Then
\begin{equation}
\label{eq:R_h_g}
\begin{array}{lll}
&\mathcal{R}_{||}(\mathcal{F'})\\
=&\mathcal{R}_{||}(h\circ g) \\
=& \mathcal{R}_{||}(h'\circ g+h(0))\\
\le &\mathcal{R}_{||}(h'\circ g) + \frac{2|h(0)|}{\sqrt{n}} \textrm{ (Theorem 12 in \cite{bartlett2003rademacher})}\\
\le& 2L \mathcal{R}_{||}(g)+ \frac{2|h(0)|}{\sqrt{n}} \textrm{ (Theorem 12 in \cite{bartlett2003rademacher})}
\end{array}
\end{equation}
Now we bound $\mathcal{R}_{||}(g)$:
\begin{equation}
\label{eq:R_g}
\begin{array}{lll}
&\mathcal{R}_{||}(g)\\
=&\mathbb{E}[\sup_{g\in{\mathcal{G}}}|\frac{2}{n}\sum_{i=1}^n\sigma_i\mb{w}^T\mb{x_i}|]\\
\le &\frac{2}{n}\mathbb{E}[\sup_{g\in\mathcal{G}}\|\mb{w}\|_2\cdot\|\sum_{i=1}^n\sigma_i\mb{x_i}\|]\\
\le &\frac{2C_3}{n}\mathbb{E}[\|\sum_{i=1}^n\sigma_i\mb{x_i}\|_2]\\
=&\frac{2C_3}{n}\mathbb{E}_{\mb{x}}[\mathbb{E}_{\sigma}[\|\sum_{i=1}^n\sigma_i\mb{x_i}\|_2] ]\\
\le &\frac{2C_3}{n}\mathbb{E}_{\mb{x}}[\sqrt{\mathbb{E}_{\sigma}[\|\sum_{i=1}^n\sigma_i\mb{x_i}\|_2^2]} ]\textrm{ (concavity of $\sqrt{\cdot}$)}\\
=&\frac{2C_3}{n}\mathbb{E}_{\mb{x}}[\sqrt{\mathbb{E}_{\sigma}[\sum_{i=1}^n \sigma_i^2\mb{x_i}^2]}]\textrm{ ($\forall i\neq j\ \sigma_i\ci\sigma_j$)}\\
=&\frac{2C_3}{n}\mathbb{E}_{\mb{x}}[\sqrt{\sum_{i=1}^n \mb{x_i}^2}]\\
\le& \frac{2C_1C_3}{\sqrt{n}}
\end{array}
\end{equation}
Putting Eq.(\ref{eq:R_h_g}) and Eq.(\ref{eq:R_g}) together, we have $\mathcal{R}_{||}(\mathcal{F'})\le \frac{4LC_1C_3}{\sqrt{n}}+\frac{2|h(0)|}{\sqrt{n}}$. Plugging into $\mathcal{R}_n(\mathcal{F})\le\frac{\sqrt{m}C_{4}}{2}\mathcal{R}_{||}(\mathcal{F'})$ completes the proof.
\end{proof}

In addition, we need the following bound of $|f(\mb{x})|$.
\begin{lemma}
With probability at least $\tau$
\begin{equation}
\sup_{\mb{x},f}|f(\mb{x})| \le \sqrt{\mathcal{J}}
\end{equation}
where $\mathcal{J}=mC_4^2h^2(0)+L^2C_1^2C_3^2C_4^2((m-1)\cos\theta+1) + 2\sqrt{m}C_1C_3C_4^2L|h(0)|\sqrt{(m-1)\cos\theta+1}$.
\end{lemma}
\begin{proof}
Let $\mb{\alpha} = [\alpha_1,\cdots,\alpha_m]^T$, $W=[\mb{w_1},\cdots,\mb{w_m}]$, $\mb{h} = [h(\mb{w_1}^T\mb{x}), \cdots, h(\mb{w_m}^T\mb{x})]^T$, then we have
\begin{equation}
\label{eq:f_square}
\begin{array}{lll}
&&f^2(\mb{x})\\
&=& (\sum_{j=1}^m \alpha_j h(\mb{w_j}^T \mb{x}))^2\\
&=& (\mb{\alpha}\cdot \mb{h})^2\\
&\le& (\|\mb{\alpha}\|_2 \|\mb{h}\|_2)^2\\
&\le& C_4^2 \|\mb{h}\|_2^2
\end{array}
\end{equation}

Now we want to derive an upper bound for $\|\mb{h}\|_2$. As $h(t)$ is L-Lipschitz, $|h(\mb{w_j}^T \mb{x})|\le L|\mb{w_j}^T\mb{x}| + |h(0)|$. Therefore
\begin{equation}
\label{eq:bound_h}
\begin{array}{lll}
&& \|\mb{h}\|_2^2\\
&=& \sum_{j=1}^m h^2(\mb{w_j}^T \mb{x})\\
&\le& \sum_{j=1}^m (L|\mb{w_j}^T\mb{x}| + |h(0)|)^2\\
&=& \sum_{j=1}^m h^2(0)+L^2(\mb{w_j}^T\mb{x})^2 + 2L|h(0)||\mb{w_j}^T\mb{x}|\\
&=& mh^2(0)+L^2\|W^T\mb{x}\|_2^2 + 2L|h(0)||W^T\mb{x}|_1\\
&\le& mh^2(0)+L^2\|W^T\mb{x}\|_2^2 + 2\sqrt{m}L|h(0)|\|W^T\mb{x}\|_2\\
&\le& mh^2(0)+L^2\|W^T\|_{op}^2\|\mb{x}\|_2^2 \\
  &&+ 2\sqrt{m}L|h(0)|\|W^T\|_{op}\|\mb{x}\|_2\\
&=& mh^2(0)+L^2\|W\|_{op}^2\|\mb{x}\|_2^2 \\
&&+ 2\sqrt{m}L|h(0)|\|W\|_{op}\|\mb{x}\|_2\\
&\le& mh^2(0)+L^2C_1^2\|W\|_{op}^2 + 2\sqrt{m}C_1L|h(0)|\|W\|_{op}\\
\end{array}
\end{equation}
where $\|\cdot\|_{op}$ denotes the operator norm. We can make use of the lower bound of $\rho(\mb{w_j}, \mb{w_k})$ for $j\neq k$ to get a bound for $\|W\|_{op}$:
\begin{equation}
\begin{array}{lll}
&&\|W\|_{op}^2\\
&=&\sup_{\|u\|_2=1} \|W u\|_2^2\\
&=& \sup_{\|u\|_2=1} (u^T W^T W u)\\
&=& \sup_{\|u\|_2=2} \sum_{j=1}^{m} \sum_{k=1}^{m} u_j u_k \mb{w_j}\cdot \mb{w_k}\\
&\le& \sup_{\|u\|_2=2} \sum_{j=1}^{m} \sum_{k=1}^{m} \\
&&|u_j| |u_k| |\mb{w_j}||\mb{w_k}|\cos(\rho(\mb{w_j},\mb{w_k}))\\
&\le& C_3^2\sup_{\|u\|_2=2} \sum_{j=1}^{m} \sum_{k=1,k\neq j}^{m} \\
&&|u_j| |u_k| \cos\theta + \sum_{j=1}^{m}|u_j|^2 \\
&&\text{(with probability at least $\tau$)}\\
\end{array}
\end{equation}
Define $u'=[|u_1|, \cdots, |u_{m^p}|]^T$, $Q\in \mathbb{R}^{m^p\times m^p}$: $Q_{jk}=\cos\theta$ for $j\neq k$ and $Q_{jj}=1$, then $\|u'\|_2 = \|u\|$ and
\begin{equation}
\begin{array}{lll}
&&\|W\|_{op}^2\\
&\le& C_3^2\sup_{\|u\|_2=2} u'^T Q u'\\
&\le& C_3^2\sup_{\|u\|_2=2} \lambda_1(Q)\|u'\|_2^2\\
&\le& C_3^2 \lambda_1(Q)
\end{array}
\end{equation}
where $\lambda_1(Q)$ is the largest eigenvalue of $Q$. By simple linear algebra we can get $\lambda_1(Q) = (m-1)\cos\theta + 1$, so
\begin{equation}
\|W\|_{op}^2 \le ((m-1)\cos\theta + 1)C_3^2
\end{equation}
Substitute to Eq.(\ref{eq:bound_h}), we have
\begin{multline}
 \|\mb{h}\|_2^2\le mh^2(0)+L^2C_1^2C_3^2((m-1)\cos\theta+1) +\\ 2\sqrt{m}C_1C_3L|h(0)|\sqrt{(m-1)\cos\theta+1}
\end{multline}
Substitute to Eq.(\ref{eq:f_square}):
\begin{equation}
\label{eq:f_square_2}
\begin{array}{lll}
&&f^2(\mb{x})\\
&\le& mC_4^2h^2(0)+L^2C_1^2C_3^2C_4^2((m-1)\cos\theta+1) +\\ &&2\sqrt{m}C_1C_3C_4^2L|h(0)|\sqrt{(m-1)\cos\theta+1}
\end{array}
\end{equation}
In order to simplify our notations, define 
\begin{multline}
\mathcal{J}=mC_4^2h^2(0)+L^2C_1^2C_3^2C_4^2((m-1)\cos\theta+1) +\\ 2\sqrt{m}C_1C_3C_4^2L|h(0)|\sqrt{(m-1)\cos\theta+1}
\end{multline}
Then $\sup_{\mb{x},f}|f(\mb{x})|\le \sqrt{\sup_{\mb{x},f}f^2(\mb{x})} = \sqrt{\mathcal{J}}$. Proof completes.
\end{proof}

Given these lemmas, we proceed to prove Theorem \ref{thm:est_err}. 
The Rademacher complexity $\mathcal{R}_n(\mathcal{A})$ of $\mathcal{A}$ is
\begin{equation}
\label{eq:ra}
\begin{array}{lll}
&&\mathcal{R}_n(\mathcal{A})\\
&=&\mathbb{E}[\textrm{sup}_{f\in \mathcal{F}}\frac{1}{n}\sum_{i=1}^{n}\sigma_i \ell(f(\mb{x}),y)]
\end{array}
\end{equation}
$\ell(\cdot,y)$ is Lipschitz continuous with respect to the first argument, and the constant $L$ is $\sup_{\mb{x},y,f}|f(\mb{x})-y|\le2\sup_{\mb{x},y,f}(|f(\mb{x})|+|y|) = 2(\sqrt{\mathcal{J}}+C_2)$. Applying the composition property of Rademacher complexity, we have
\begin{equation}
\mathcal{R}_n(\mathcal{A}) \le 2(\sqrt{\mathcal{J}}+C_2)\mathcal{R}_n(\mathcal{F})
\end{equation}
Using Lemma \ref{lem:rc_f}, we have
\begin{equation}
\label{eq:R_A}
\mathcal{R}_n(\mathcal{A}) \le 2(\sqrt{\mathcal{J}}+C_2)(\frac{2 L C_{1}C_{3}C_{4}\sqrt{m}}{\sqrt{n}} + \frac{C_{4}|h(0)|\sqrt{m}}{\sqrt{n}})
\end{equation}
Note that $\sup_{\mb{x},y,f}|\ell(f(\mb{x}),y)|\le (\sqrt{\mathcal{J}}+C_2)^2$, and plugging Eq.(\ref{eq:R_A}) into Lemma \ref{lem:rc_bd} completes the proof.

\subsubsection{Extensions}
In the above analysis, we consider a simple neural network described in Section \ref{sec:setup}. In this section, we present how to extend the analysis to more complicated cases, such as neural networks with multiple hidden layers, other loss functions and multiple outputs. 

\paragraph{Multiple Hidden Layers}
The analysis can be extended to multiple hidden layers by recursively applying the composition property of Rademacher complexity to the hypothesis set.

We define the hypothesis set $\mathcal{F}^P$ for neural network with $P$ hidden layers in a recursive manner:
\begin{equation}
\begin{array}{lll}
\mathcal{F}^0 &=& \{f^0|f^0(\mb{x})=\mb{w}^0\cdot\mb{x}\}\\
\mathcal{F}^1 &=& \mathcal{F}=\{f^1|f^1(\mb{x})=\sum_{j=1}^{m^0}\mb{w_j}^1h(f_j^0(\mb{x})),\\
&&f_j^0\in\mathcal{F}^0\}\\
\mathcal{F}^{p} &=&\{f^p|f^p(\mb{x})=\sum_{j=1}^{m^{p-1}}\mb{w_j}^ph(f_j^{p-1}(\mb{x})),\\
&&f_j^{p-1}\in\mathcal{F}^{p-1}\}(l=2,\cdots,P)
\end{array}
\end{equation}
where we assume there are $m^p$ units in hidden layer $p$ and $\mb{w_j^p}$ is the connecting weight from the j-th unit in hidden layer $p-1$ to $p$. (we index hidden layers from 0, $\mb{w^0}$ is the connecting weight from input to hidden layer 0). When $P=1$ the above definition recovers the one-hidden-layer case in Section \ref{sec:setup} if we treat $\mb{w^1}$ as $\alpha$. We make similar assumptions as Section \ref{sec:setup}: $h(\cdot)$ is L-Lipschitz, $\|\mb{x}\|_2\le C_1$, $\|\mb{w}^p\|_2\le C_3^p$. We also assume that the pairwise angles of the connecting weights $\rho(\mb{w_j^p}, \mb{w_k^p})$ for $j\neq k$ are lower bounded by $\theta^p$ with probability at least $\tau^p$. Under these assumptions, we have the following result:
\begin{thm}
\label{thm:est_err_gen}
For a neural network with $P$ hidden layers, with probability at least $(1-\delta)\prod_{p=0}^{P-1}\tau^p$
\begin{equation}
\label{eq:est_err_gen}
\begin{array}{lll}
&&L(\hat{f})-L(f^*)\\
&\leq& 8(\sqrt{\mathcal{J}^p}+C_2)(\frac{(2L)^PC_1C_3^0}{\sqrt{n}}\prod_{p=0}^{P-1}\sqrt{m^p}C_3^p\\
&&+\frac{|h(0)|}{\sqrt{n}}\sum_{p=0}^{P-1}(2L)^{P-1-p}\prod_{j=p}^{P-1}\sqrt{m^j}C_3^j)\\
&&+ (\sqrt{\mathcal{J}^p}+C_2)^2\sqrt{\frac{2\log(2/\delta)}{n}}
\end{array}
\end{equation}
where 
\begin{equation}
\begin{array}{lll}
&&\mathcal{J}^0 = C_1^2((m^0-1)\cos\theta^0+1)\\
&&\mathcal{J}^{p} = (C_3^p)^2((m^p-1)\cos\theta^p+1)L^2\mathcal{J}^{p-1}\\
&&+2(C_3^p)^2L|h(0)|\sqrt{m^{p-1}}((m^p-1)\cos\theta^p+1)\sqrt{\mathcal{J}^{p-1}}+\\
&&(C_3^p)^2((m^p-1)\cos\theta^p+1)m^{p-1}h^2(0)(p=1,\cdots,P)
\end{array}
\end{equation}
\end{thm}

When $P=1$, Eq.(\ref{eq:est_err_gen}) reduces to the estimation error bound of neural network with one hidden layer. Note that the right hand side is a decreasing function w.r.t $\theta^p$, hence making the hidden units in each hidden layer to be diverse can reduce the estimation error bound of neural networks with multiple hidden layers.

In order to prove Theorem \ref{thm:est_err_gen}, we first bound the Rademacher complexity of the hypothesis set $\mathcal{F}^{P}$:
\begin{lemma}
\label{lem:rc_f_gen}
Let $\mathcal{R}_n(\mathcal{F}^P)$ denote the Rademacher complexity of the hypothesis set $\mathcal{F}^P$, then
\begin{multline}
\mathcal{R}_n(\mathcal{F}^P) \le \frac{(2L)^PC_1C_3^0}{\sqrt{n}}\prod_{p=0}^{P-1}\sqrt{m^p}C_3^p\\
+\frac{|h(0)|}{\sqrt{n}}\sum_{p=0}^{P-1}(2L)^{P-1-p}\prod_{j=p}^{P-1}\sqrt{m^j}C_3^j
\end{multline}
\end{lemma}
\begin{proof}
Notice that $\mathcal{R}_n(\mathcal{F}^P))\le \frac{1}{2}\mathcal{R}_{||}(\mathcal{F}^P)$:
\begin{equation}
\label{eq:R_rel}
\begin{array}{lll}
&&\mathcal{R}_n (\mathcal{F}^P) \\ &=&\mathbb{E}[\sup_{f\in\mathcal{F}^P}\frac{1}{n}\sum_{i=1}^n\sigma_if(\mb{x_i})]\\
&\le& \mathbb{E}[\sup_{f\in\mathcal{F}^P}|\frac{1}{n}\sum_{i=1}^n\sigma_if(\mb{x_i})|]\\
&=& \frac{1}{2}\mathcal{R}_{||}(\mathcal{F}^P)
\end{array}
\end{equation}
So we can bound $\mathcal{R}_n(\mathcal{F}^P))$ by bounding $\frac{1}{2}\mathcal{R}_{||}(\mathcal{F}^P)$. We bound $\frac{1}{2}\mathcal{R}_{||}(\mathcal{F}^p)$ recursively: $\forall p=1,\cdots,P$, we have
\begin{equation}
\label{eq:R_mult_rec}
\begin{array}{lll}
&&\mathcal{R}_{||}(\mathcal{F}^p)\\
&=&\mathbb{E}[\sup_{f\in\mathcal{F}^p}|\frac{2}{n}\sum_{i=1}^n\sigma_if(\mb{x_i})|]\\
&=&\mathbb{E}[\sup_{f_j\in\mathcal{F}^{p-1}}|\frac{2}{n}\sum_{i=1}^n\sigma_i\sum_{j=1}^{m^{l-1}}\mb{w_j}^lh(f_j(\mb{x_i}))|]\\
&\le& \sqrt{m^{p-1}}C_3^p \mathbb{E}[\sup_{f_j\in\mathcal{F}^{p-1}}|\frac{2}{n}\sum_{i=1}^n\sigma_ih(f_j(\mb{x_i}))|]\\
&\le& \sqrt{m^{p-1}}C_3^p(2L\mathcal{R}_{||}(\mathcal{F}^{p-1})+\frac{2|h(0)|}{\sqrt{n}})
\end{array}
\end{equation}
where the last two steps are similar to the proof of Lemma \ref{lem:rc_f}. Applying the inequality in Eq.(\ref{eq:R_mult_rec}) recursively, and noting from the proof of Lemma \ref{lem:rc_f} that $\mathcal{R}_{||}(\mathcal{F}^0) \le \frac{2C_1C_3^0}{\sqrt{n}}$ we have
\begin{equation}
\begin{array}{lll}
&\mathcal{R}_{||}(\mathcal{F}^P) &\le \frac{2(2L)^PC_1C_3^0}{\sqrt{n}}\prod_{p=0}^{P-1}\sqrt{m^p}C_3^p\\
&&+\frac{2|h(0)|}{\sqrt{n}}\sum_{p=0}^{P-1}(2L)^{P-1-p}\prod_{j=p}^{P-1}\sqrt{m^j}C_3^j
\end{array}
\end{equation}
Plugging into Eq.(\ref{eq:R_rel}) completes the proof.
\end{proof}
In addition, we need the following bound.
\begin{lemma}
With probability at least $\prod_{p=0}^{P-1}\tau^p$, 
$\sup_{\mb{x}, f^P\in\mathcal{F}^p}|f^P(\mb{x})| \le \sqrt{\mathcal{J}^P}$, where
\begin{equation}
\begin{array}{lll}
&&\mathcal{J}^0 = C_1^2((m^0-1)\cos\theta^0+1)\\
&&\mathcal{J}^{p} = (C_3^p)^2((m^p-1)\cos\theta^p+1)L^2\mathcal{J}^{p-1}\\
&&+2(C_3^p)^2L|h(0)|\sqrt{m^{p-1}}((m^p-1)\cos\theta^p+1)\sqrt{\mathcal{J}^{p-1}}\\
&&+(C_3^p)^2((m^p-1)\cos\theta^p+1)m^{p-1}h^2(0)(1,\cdots,P)
\end{array}
\end{equation}
\end{lemma}
\begin{proof}
For a given neural network, we denote the outputs of the p-th hidden layer before applying the activation function as $v^p$:
\begin{equation}
\begin{array}{lll}
v^0 &=& [\mb{w_1^0}^T\mb{x}, \cdots, \mb{w_{m^0}^0}\mb{x}]^T\\
v^p &=& [\sum_{j=1}^{m^{p-1}}\mb{w_{j,1}^p}h(v_j^{p-1}), \cdots,\\
&&\sum_{j=1}^{m^{p-1}}\mb{w_{j,m^p}^p}h(v_j^{p-1})]^T (p=1,\cdots,P)
\end{array}
\end{equation}
where $\mb{w_{j,i}^p}$ is the connecting weight from the $j$-th unit of the hidden layer $p-1$ to the $i$-th unit of the hidden layer $p$.

To facilitate the derivation of bounds, we also denote
\begin{equation}
\mb{w_i^p} = [\mb{w_{1,i}^p}, \cdots, \mb{w_{m^{p-1},i}^p}]^T
\end{equation}
and
\begin{equation}
\mb{h^p} = [h(v_1^{p-1}), \cdots, h(v_{m^{p-1}}^{p-1})]^T
\end{equation}
where $v_i^{p-1}$ is the $i$-th element of $v^{p-1}$.

Using the above notations, we can write $v^p$ as
\begin{equation}
v^p = [\mb{w_1^p}\cdot\mb{h^p}, \cdots, \mb{w_{m^p}^p}\cdot\mb{h^p}]^T
\end{equation}

Hence we can bound the $L_2$ norm of $v^p$ recursively:
\begin{equation}
\begin{array}{lll}
&&\|v^p\|_2^2 = \sum_{i=1}^{m^p}(\mb{w_i^p}\cdot\mb{h^p})^2
\end{array}
\end{equation}
Denote $W = [\mb{w_1^p}, \cdots, \mb{w_{m^p}^p}]$, then
\begin{equation}
\begin{array}{lll}
\label{eq:v_p_recur}
&&\|v^p\|_2^2 \\
&=& \|W^T \mb{h^p}\|_2^2\\
&\le& \|W^T\|_{op}^2 \|h^p\|_2^2\\
&=& \|W\|_{op}^2 \|h^p\|_2^2
\end{array}
\end{equation}
where $\|\cdot\|_{op}$ denotes the operator norm.

We can make use of the lower bound of $\rho(\mb{w_j^p}, \mb{w_k^p})$ for $j\neq k$ to get a bound for $\|W\|_{op}$:
\begin{equation}
\begin{array}{lll}
&&\|W\|_{op}^2\\
&=&\sup_{\|u\|_2=1} \|W u\|_2^2\\
&=& \sup_{\|u\|_2=1} (u^T W^T W u)\\
&=& \sup_{\|u\|_2=2} \sum_{j=1}^{m^p} \sum_{k=1}^{m^p} u_j u_k \mb{w_j^p}\cdot \mb{w_k^p}\\
&\le& \sup_{\|u\|_2=2} \sum_{j=1}^{m^p} \sum_{k=1}^{m^p} \\
&&|u_j| |u_k| |\mb{w_j^p}||\mb{w_k^p}|\cos(\rho(\mb{w_j^p},\mb{w_k^p}))\\
&\le& (C_3^p)^2\sup_{\|u\|_2=2} \sum_{j=1}^{m^p} \sum_{k=1,k\neq j}^{m^p} \\
&&|u_j| |u_k| \cos\theta^p + \sum_{j=1}^{m^p}|u_j|^2\\
&&(\text{with probability at least $\prod_{p=0}^{P-1}\tau^p$})\\
\end{array}
\end{equation}
Define $u'=[|u_1|, \cdots, |u_{m^p}|]^T$, $Q\in \mathbb{R}^{m^p\times m^p}$: $Q_{jk}=\cos\theta^p$ for $j\neq k$ and $Q_{jj}=1$, then $\|u'\|_2 = \|u\|$ and
\begin{equation}
\begin{array}{lll}
&&\|W\|_{op}^2\\
&\le& (C_3^p)^2\sup_{\|u\|_2=2} u'^T Q u'\\
&\le& (C_3^p)^2\sup_{\|u\|_2=2} \lambda_1(Q)\|u'\|_2^2\\
&\le& (C_3^p)^2 \lambda_1(Q)
\end{array}
\end{equation}
where $\lambda_1(Q)$ is the largest eigenvalue of $Q$. By simple linear algebra we can get $\lambda_1(Q) = (m^p-1)\cos\theta^p + 1$, so
\begin{equation}
\label{eq:bd_wop}
\|W\|_{op}^2 \le ((m^p-1)\cos\theta^p + 1)(C_3^p)^2
\end{equation}
Substituting Eq.(\ref{eq:bd_wop}) back to Eq.(\ref{eq:v_p_recur}), we have
\begin{equation}
\begin{array}{l}
\label{eq:v_p_recur_2}
\|v^p\|_2^2 \le (C_3^p)^2((m^p-1)\cos\theta^p + 1) \|h^p\|_2^2
\end{array}
\end{equation}
Then we make use of the Lipschitz-continuous property of $h(t)$ to further bound $\|h^p\|_2^2$:
\begin{equation}
\label{eq:bd_hp}
\begin{array}{lll}
&&\|h^p\|_2^2\\
&=& \sum_{j=1}^{m^{p-1}} h^2(v_j^{p-1}) \\
&\le& \sum_{j=1}^{m^{p-1}} (|h(0)|+L|v_j^{p-1}|)^2 \\
&=& \sum_{j=1}^{m^{p-1}} h^2(0)+L^2(v_j^{p-1})^2 + 2L|h(0)||v_j^{p-1}| \\
&=& m^{p-1}h^2(0)+L^2\|v^{p-1}\|_2^2 + 2L|h(0)|\|v^{p-1}\|_1\\
&\le& m^{p-1}h^2(0)+L^2\|v_j^{p-1}\|_2^2 + 2L|h(0)|\sqrt{m^{p-1}}\|v^{p-1}\|_2\\
\end{array}
\end{equation}
Substituting Eq.(\ref{eq:bd_hp}) to Eq.(\ref{eq:v_p_recur_2}), we have
\begin{multline}
\|v^p\|_2^2 \le (C_3^p)^2((m^p-1)\cos\theta^p+1)L^2\|v_j^{p-1}\|_2^2\\
+2(C_3^p)^2L|h(0)|\sqrt{m^{p-1}}((m^p-1)\cos\theta^p+1)\|v^{p-1}\|_2\\
+(C_3^p)^2((m^p-1)\cos\theta^p+1)m^{p-1}h^2(0)
\end{multline}
And noticing that $\|v_0\|_2^2\le((m^0-1)\cos\theta^0+1)\|\mb{x}\|_2^2\le C_1^2((m^0-1)\cos\theta^0+1)$, we can bound $\|v^p\|$ recursively now.
Denote
\begin{equation}
\begin{array}{lll}
&&\mathcal{J}^0 = C_1^2((m^0-1)\cos\theta^0+1)\\
&&\mathcal{J}^{p} = (C_3^p)^2((m^p-1)\cos\theta^p+1)L^2\mathcal{J}^{p-1}\\
&&+2(C_3^p)^2L|h(0)|\sqrt{m^{p-1}}((m^p-1)\cos\theta^p+1)\sqrt{\mathcal{J}^{p-1}}+\\
&&(C_3^p)^2((m^p-1)\cos\theta^p+1)m^{p-1}h^2(0)(p=1,\cdots,P)
\end{array}
\end{equation}
then $\|v^p\|_2^2 \le \mathcal{J}^{p}$ and $\mathcal{J}^p$ decreases when $\theta^i(i=0,\cdots,p)$ increases.

Now we are ready to bound $\sup_{\mb{x},f^P\in\mathcal{F}^P}|f^P(\mb{x})|$:
\begin{equation}
\begin{array}{lll}
&&\sup_{\mb{x},f^P\in\mathcal{F}^P}|f^P(\mb{x})|\\
&=& \sup_{\mb{x},f^P\in\mathcal{F}^P} |v^P|\\
&\le& \sqrt{\mathcal{J}^P}
\end{array}
\end{equation}
\end{proof}

Given these lemmas, we proceed to prove Theorem \ref{thm:est_err_gen}. 
The Rademacher complexity $\mathcal{R}_n(\mathcal{A})$ of $\mathcal{A}$ is
\begin{equation}
\label{eq:ra}
\begin{array}{l}
\mathcal{R}_n(\mathcal{A})=\mathbb{E}[\textrm{sup}_{f\in \mathcal{F}}\frac{1}{n}\sum_{i=1}^{n}\sigma_i \ell(f(\mb{x_i}),y)]
\end{array}
\end{equation}
$\ell(\cdot,y)$ is Lipschitz continuous with respect to the first argument, and the constant $L$ is $\sup_{\mb{x},y,f}|f(\mb{x})-y|\le2\sup_{\mb{x},y,f}(|f(\mb{x})|+|y|) = 2(\sqrt{\mathcal{J}}+C_2)$. Applying the composition property of Rademacher complexity, we have
\begin{equation}
\mathcal{R}_n(\mathcal{A}) \le 2(\sqrt{\mathcal{J}}+C_2)\mathcal{R}_n(\mathcal{F})
\end{equation}
Using Lemma \ref{lem:rc_f_gen}, we have
\begin{multline}
\label{eq:R_A_gen}
\mathcal{R}_n(\mathcal{A}) \le 2(\sqrt{\mathcal{J}}+C_2)(\frac{(2L)^PC_1C_3^0}{\sqrt{n}}\prod_{p=0}^{P-1}\sqrt{m^p}C_3^p\\
+\frac{|h(0)|}{\sqrt{n}}\sum_{p=0}^{P-1}(2L)^{P-1-p}\prod_{j=p}^{P-1}\sqrt{m^j}C_3^j)
\end{multline}
Note that $\sup_{\mb{x},y,f}|\ell(f(\mb{x}),y)|\le (\sqrt{\mathcal{J}}+C_2)^2$, and plugging Eq.(\ref{eq:R_A_gen}) into Lemma \ref{lem:rc_bd} completes the proof.
\paragraph{Other Loss Functions} Other than regression, a more popular application of neural network is classification. For binary classification, the most widely used loss functions are logistic loss and hinge loss. Estimation error bounds similar to that in Theorem \ref{thm:est_err} can also be derived for these two loss functions.
\begin{lemma}
\label{lem:logistic_ub}
Let the loss function $\ell(f(x),y)=\log(1+\exp(-yf(x)))$ be the logistic loss where $y\in\{-1,1\}$, then with probability at least $(1-\delta)\tau$
\begin{equation}
\begin{array}{lll}
&&L(\hat{f})-L(f^*)\\
&\leq& \frac{4}{1+\exp(-\sqrt{\mathcal{J}})}(2LC_1C_3C_4+C_4|h(0)|)\frac{\sqrt{m}}{\sqrt{n}}\\
&&+ \log (1+\exp(\sqrt{J}))\sqrt{\frac{2\log(2/\delta)}{n}}
\end{array}
\end{equation}
\end{lemma}
\begin{proof}
\begin{equation}
|\frac{\partial l(f(x),y)}{\partial f}| = \frac{\exp(-y f(x))}{1+\exp(-yf(x))}=\frac{1}{1+\exp(yf(x))}
\end{equation}
As $|\frac{1}{1+\exp(yf(x))}|\le \frac{1}{1+\exp(-\sup_{f,x} |f(x)|)}=\frac{1}{1+\exp(-\sqrt{\mathcal{J}})}$, we have proved that the Lipschitz constant L of $\ell(\cdot,y)$ can be bounded by $\frac{1}{1+\exp(-\sqrt{\mathcal{J}})}.$

And the loss function $\ell(f(x),y)$ can be bounded by
\begin{equation}
|\ell(f(x),y)|\le \log(1+\exp(\sqrt{J}))
\end{equation}

Similar to the proof of Theorem \ref{thm:est_err}, we can finish the proof by applying the composition property of Rademacher complexity, Lemma \ref{lem:rc_f} and Lemma \ref{lem:rc_bd}.
\end{proof}

\begin{lemma}
\label{lem:angle}
Let $\ell(f(x),y)=\max(0,1-yf(x))$ be the hinge loss where $y\in\{-1,1\}$, then with probability at least $(1-\delta)\tau$
\begin{equation}
\begin{array}{lll}
&&L(\hat{f})-L(f^*)\\
&\leq& 4(2LC_1C_3C_4+C_4|h(0)|)\frac{\sqrt{m}}{\sqrt{n}}\\
&&+ (1+\sqrt{J})\sqrt{\frac{2\log(2/\delta)}{n}}
\end{array}
\end{equation}
\end{lemma}
\begin{proof}
Given $y$, $\ell(\cdot,y)$ is Lipschitz with constant 1. And the loss function can be bounded by
\begin{equation}
|\ell(f(x),y)|\le 1+\sqrt{J}
\end{equation}

The proof can be completed using similar proof of Lemma \ref{lem:logistic_ub}.
\end{proof}


\paragraph{Multiple Outputs} The analysis can be also extended to neural networks with multiple outputs, provided the loss function factorizes over the dimensions of the output vector. Let $\mb{y}\in \mathrm{R}^K$ denote the target output vector, $\mb{x}$ be the input feature vector and $\ell(f(\mb{x}),\mb{y})$ be the loss function. If $\ell(f(\mb{x}),\mb{y})$ factorizes over $k$, i.e., $\ell(f(\mb{x}),\mb{y})=\sum_{k=1}^{K}\ell'(f(\mb{x})_k,y_k)$, then we can perform the analysis for each $\ell'(f(\mb{x})_k,y_k)$ as that in Section \ref{sec:thm1_proof} separately and sums the estimation error bounds up to get the error bound for $\ell(f(\mb{x}),\mb{y})$. Here we present two examples. For multivariate regression, the loss function $\ell(f(\mb{x}),\mb{y})$ is a squared loss: $\ell(f(\mb{x}),\mb{y})=\|f(\mb{x})-\mb{y}\|_2^2$, where $f(\cdot)$ is the prediction function. This squared loss can be factorized as $\|f(\mb{x})-\mb{y}\|_2^2=\sum_{k=1}^{K}(f(\mb{x})_k-y_k)^2$. We can obtain an estimation error bound for each $(f(\mb{x})_k-y_k)^2$ according to Theorem \ref{thm:est_err}, then sum these bounds together to get the bound for $\|f(\mb{x})-\mb{y}\|_2^2$. 

For multiclass classification, the commonly used loss function is cross-entropy loss: $\ell(f(x),\mb{y})=-\sum_{k=1}^{K}y_k\log a_k$, where $a_k=\frac{\exp(f(\mb{x})_k)}{\sum_{j=1}^{K}\exp(f(\mb{x})_j)}$. We can also derive error bounds similar to that in Theorem \ref{thm:est_err} by using the composition property of Rademacher complexity. First we need to find the Lipschitz constant:
\begin{lemma}
\label{lem:cross_ent}
Let $\ell(\mb{x},\mb{y},f)$ be the cross-entropy loss, then for any $f$, $f'$
\begin{multline}
|\ell(f(x),y)-\ell(f'(x),y)|\le\\ \frac{K-1}{K-1+\exp(-2\sqrt{\mathcal{J}})}\sum_{k=1}^K|f(x)_k-f'(x)_k|
\end{multline}
\end{lemma}
\begin{proof}
Note that $\mb{y}$ is a 1-of-K coding vector where exactly one element is 1 and all others are 0. Without loss of generality, we assume $y_{k'}=1$ and $y_{k}=0$ for $k\neq k'$. Then
 \begin{equation}
 \ell(f(x),\mb{y})= -\log\frac{\exp(f(\mb{x})_{k'})}{\sum_{j=1}^{K}\exp(f(\mb{x})_{j})}
 \end{equation}
 
Hence for $k\neq k'$ we have
\begin{equation}
\begin{array}{lll}
&&|\frac{\partial l(f(x), y)}{\partial f(x)_{k}}|\\
&=& \frac{1}{1+\sum_{j\neq k'}\exp(f(x)_j)}\\
&\leq& \frac{1}{1+(K-1)\exp(-2\sqrt{\mathcal{J}})}
\end{array}
\end{equation}
and for $k'$ we have
\begin{equation}
\begin{array}{lll}
&&|\frac{\partial l(f(x), y)}{\partial f(x)_{k'}}|\\
&=& \frac{\sum_{j\neq k'}\exp(f(x)_j)}{1+\sum_{j\neq k'}\exp(f(x)_j)}\\
&\leq& \frac{K-1}{K-1+\exp(-2\sqrt{\mathcal{J}})}
\end{array}
\end{equation}
As $\frac{K-1}{K-1+\exp(-2\sqrt{\mathcal{J}})}\ge \frac{1}{1+(K-1)\exp(-2\sqrt{\mathcal{J}})}$, we have proved that for any $k$, $|\frac{\partial l(f(x), y)}{\partial f(x)_{k}}|\leq \frac{K-1}{K-1+\exp(-2\sqrt{\mathcal{J}})}$. Therefore
\begin{equation}
\|\nabla_{f(x)} \ell(f(x),y)\|_\infty \le \frac{K-1}{K-1+\exp(-2\sqrt{\mathcal{J}})}
\end{equation}

Using mean value theorem, for any $f$, $f'$, $\exists \xi$ such that
\begin{equation}
\begin{array}{lll}
&&|\ell(f(x),y)-\ell(f'(x),y)|\\
&=& \nabla_{\xi} \ell(\xi,y) \cdot (f(x)-f'(x))\\
&\le& \|\nabla_{f(x)} \ell(f(x),y)\|_\infty \|f(x)-f'(x)\|_1\\
&\le& \frac{K-1}{K-1+\exp(-2\sqrt{\mathcal{J}})}\sum_{k=1}^K|f(x)_k-f'(x)_k|
\end{array}
\end{equation}
\end{proof}
With Lemma \ref{lem:cross_ent}, we can get the Rademacher complexity of cross entropy loss by performing the Rademacher complexity analysis for each $f(x)_k$ as that in Section \ref{sec:thm1_proof} separately, and multiplying the sum of them by $\frac{K-1}{K-1+\exp(-2\sqrt{\mathcal{J}})}$ to get the Rademacher complexity of $\ell(f(x),y)$. And as the loss function can be bounded by
\begin{equation}
|\ell(f(x),y)|\le \log (1+(K-1)\exp(2\sqrt{\mathcal{J}}))
\end{equation}
we can use similar proof techniques as in Theorem \ref{thm:est_err} to get the estimation error bound.

\subsection{Approximation Error}

Now we proceed to investigate how the diversity of weight vectors affects the approximation error bound. For the ease of analysis, following \cite{barron1993universal}, we assume the target function $g$ belongs to a function class with smoothness expressed in the first moment of its Fourier representation: we define function class $\Gamma_C$ as the set of functions $g$ satisfying
\begin{equation}
\int_{\|x\|_2\le C_1}|w||\tilde{g}(w)|dw\le C
\end{equation}
where $\tilde{g}(w)$ is the Fourier representation of $g(x)$ and we assume $\|x\|_2 \le C_1$ throughout this paper. We use function $f$ in $\mathcal{F} = \{f|f(x)=\sum_{j=1}^m \alpha_j h(w_j^T x)\}$ which is the NN function class defined in Section \ref{sec:setup}, to approximate $g\in \Gamma_C$. Recall the following conditions of $\mathcal{F}$:
\begin{align}
&\forall j \in \{1,\cdots,m\}, \|w_j\|_2\le C_3\\
&\|\alpha\|_2\le C_4\\
&\forall j\neq k, \rho(w_j,w_k)\ge\theta (\text{with probability at least $\tau$})
\end{align}
where the activation function $h(t)$ is the sigmoid function and we assume $\|x\|_2\le C_1$.
The following theorem states the approximation error.
\begin{thm}
\label{thm:appro2}
Given $C>0$, for every function $g\in \Gamma_C$ with $g(0)=0$, for any measure $P$, if 
\begin{align}
&C_1C_3\ge 1\\
&C_4 \ge 2\sqrt{m} C\\
&m\le2(\lfloor\frac{\frac{\pi}{2}-\theta}{\theta}\rfloor+1)
\end{align}
then with probability at least $\tau$, there is a function $f \in \mathcal{F}$ such that
\begin{equation}
\label{eq:app_eb}
\|g - f\|_L \le 2C(\frac{1}{\sqrt{n}} + \frac{1+ 2\ln C_1C_3}{C_1C_3}) + 4mCC_1C_3\sin(\frac{\theta'}{2})
\end{equation}
where $\|f\|_{L} = \sqrt{\int_{x}f^2(x)dP(x)}$, $\theta'=\min(3m\theta, \pi)$.
\end{thm}

Note that the approximation error bound in Eq.(\ref{eq:app_eb}) is an increasing function of $\theta$. Hence increasing the diversity of hidden units would hurt the approximation capability of neural networks.

\subsection{Proof}
Before proving Theorem \ref{thm:appro2}, we need the following lemma:
\begin{lemma}
\label{lem:theta_sum_bound}
For any three nonzero vectors $u_1$, $u_2$, $u_3$, let $\theta_{12} = \arccos (\frac{u_1 \cdot u_2}{\|u_1\|_2 \|u_2\|_2})$, $\theta_{23} = \arccos (\frac{u_2 \cdot u_3}{\|u_2\|_2 \|u_3\|_2})$, $\theta_{13} = \arccos (\frac{u_1 \cdot u_3}{\|u_1\|_2 \|u_3\|_2})$. We have $\theta_{13}\le\theta_{12}+\theta_{23}$.
\end{lemma}
\begin{proof}
Without loss of generality, assume $\|u_1\|_2=\|u_2\|_2=\|u_3\|_2=1$. Decompose $u_1$ as $u_1=u_{1//}+u_{1\perp}$ where $u_{1//} = c_{12} u_2$ for some $c_{12}\in\mathbb{R}$ and $u_{1\perp}\perp u_{2}$. As $u_{1}\cdot u_2 = \cos\theta_{12}$, we have $c_{12}=\cos\theta_{12}$ and $\|u_{1\perp}\|_2=\sin\theta_{12}$. 

Similarly, decompose $u_3$ as $u_3=u_{3//}+u_{3\perp}$ where $u_{3//} = c_{32} u_2$ for some $c_{32}\in\mathbb{R}$ and $u_{3\perp}\perp u_{2}$. We have $c_{23}=\cos\theta_{23}$ and $\|u_{3\perp}\|_2=\sin\theta_{23}$.

So we have
\begin{equation}
\begin{array}{lll}
&&\cos\theta_{13}\\
&=&u_1\cdot u_3\\
&=& (u_{1//}+u_{1\perp})\cdot (u_{3//}+u_{3\perp})\\
&=& u_{1//}\cdot u_{3//} + u_{1\perp}\cdot u_{3\perp}\\
&=& \cos\theta_{12}\cos\theta_{23}+ u_{1\perp}\cdot u_{3\perp}\\
&\ge& \cos\theta_{12}\cos\theta_{23} - \sin\theta_{12}\sim\theta_{23}\\
&=&\cos(\theta_{12}+\theta_{23})
\end{array}
\end{equation}
If $\theta_{12}+\theta_{23}\le\pi$, $\arccos (\cos(\theta_{12}+\theta_{23}))=\theta_{12}+\theta_{23}$. As $\arccos(\cdot)$ is monotonously decreasing, we have $\theta_{13}\le\theta_{12}+\theta_{23}$. Otherwise, $\theta_{13}\le\pi\le\theta_{12}+\theta_{23}$.
\end{proof}
In order to approximate the function class $\Gamma_C$, we first remove the constraints $\rho(w_j,w_k)\ge\theta$ and obtain an approximation error:
\begin{lemma}
\label{lem:appro2}
Let $\mathcal{F'} = \{f|f(x)=\sum_{j=1}^m \alpha_j h(w_j^T x)\}$ be the function class satisfying the following constraints:
\begin{itemize}
\item $|\alpha_j| \le 2C$
\item $\|w_j\|_2 \le C_3$
\end{itemize}
Then for every $g\in \Gamma_C$ with $g(0)=0$, $\exists f' \in \mathcal{F'}$ such that
\begin{equation}
\|g(x) - f'(x)\|_L \le 2C(\frac{1}{\sqrt{n}} + \frac{1+ 2\ln C_1C_3}{C_1C_3})
\end{equation}
\end{lemma}
\begin{proof}
Please refer to Theorem 3 in \cite{barron1993universal} for the proof. Note that the $\tau$ used in their paper is $C_1 C_3$ here. Furthermore, we omit the bias term $b$ as we can always add a dummy feature $1$ to the input $x$ to avoid using the bias term.
\end{proof}

We also need the following lemma:
\begin{lemma}
\label{lem:theta_appro}
For any $0 < \theta < \frac{\pi}{2}$, $m\le2(\lfloor\frac{\frac{\pi}{2}-\theta}{\theta}\rfloor+1)$, $(w_j')_{j=1}^m$, $\exists (w_j)_{j=1}^m$ such that
\begin{align}
&\forall j\neq k\in\{1,\cdots,m\}, \rho(w_j,w_k)\ge \theta\\
&\forall j\in\{1,\cdots,m\}, \|w_j\|_2 = \|w_j'\|_2\\
&\forall j \in \{1,\cdots,m\},\arccos (\frac{w_j \cdot w_j'}{\|w_j\|_2\|w_j'\|_2}) \le \theta'
\end{align}
where $\theta' = \min(3m\theta,{\pi})$.
\end{lemma}
\begin{proof}
To simplify our notations, let $\phi(a,b) = \arccos(\frac{a\cdot b}{\|a\|_2\|b\|_2})$. We begin our proof by considering a 2-dimensional case: Let 
\begin{equation}
k = \lfloor\frac{\frac{\pi}{2}-\theta}{\theta}\rfloor
\end{equation}
Let index set $\mathcal{I} = \{-(k+1), -k,\cdots,-1,1,2,\cdots,k+1\}$. We define a set of vectors $(e_i)_{i\in\mathcal{I}}$: $e_i = (\sin \theta_i, \cos \theta_i)$, where $\theta_i \in (-\frac{\pi}{2},\frac{\pi}{2})$ is defined as follows:
\begin{equation}
\label{eq:def_theta}
\theta_i = \mathrm{sgn}(i)(\frac{\theta}{2} + (|i|-1)\theta)
\end{equation}
From the definition we can verify the following conclusions:
\begin{align}
&\forall i\neq j\in\mathcal{I}, \rho(e_i,e_j) \ge \theta\\
&-\frac{\pi}{2}+\frac{\theta}{2}\le\theta_{-(k+1)}<-\frac{\pi}{2}+\frac{3}{2}\theta\\
&\frac{\pi}{2}-\frac{3}{2}\theta<\theta_{k+1}\le\frac{\pi}{2}-\frac{\theta}{2}\\
\end{align}
And we can further verify that $\forall i\in\mathcal{I}$, there exists different $i_1, \cdots,i_{2k+1}\in\mathcal{I}\backslash i$ such that $\phi(e_{i}, e_{i_j})\le j\theta$.

For any $e = (\sin \beta, \cos \beta)$ with $\beta\in[-\frac{\pi}{2},\frac{\pi}{2}]$, we can find $i\in\mathcal{I}$ such that $\phi(e_i, e)\le\frac{3}{2}\theta$:
\begin{itemize}
\item if $\beta\ge\theta_{k+1}$, take $i=k+1$, we have $\phi(e_i,e)\le \frac{\pi}{2}-\theta_{k+1}<\frac{3}{2}\theta$.
\item if $\beta\le\theta_{-(k+1)}$, take $i=-(k+1)$, we also have $\phi(e_i,e)\le\frac{3}{2}\theta$
\item otherwise, take $i$ = $\textrm{sgn}(\beta)\lceil\frac{\beta-\frac{\theta}{2}}{\theta}\rceil$, we also have $\phi(e_i,e)\le \theta<\frac{3}{2}\theta$.
\end{itemize} 
Recall that for any $i$, there exists different $i_1, \cdots,i_{2k+1}\in\mathcal{I}\backslash i$ such that $\phi(e_{i}, e_{i_j})\le j\theta$, and use Lemma \ref{lem:theta_sum_bound}, we can draw the conclusion that for any $e = (\sin \beta, \cos \beta)$ with $\beta\in[-\frac{\pi}{2},\frac{\pi}{2}]$, there exists different $i_1, \cdots,i_{2k+2}$ such that $\phi(e_{i}, e_{i_j})\le \frac{3}{2}\theta+(j-1)\theta=(j+\frac{1}{2})\theta$. 

For any $(w_j')_{j=1}^m$, assume $w_j' = \|w_j'\|_2(\sin\beta_j,\cos\beta_j)$, and we assume $\beta_j\in[-\frac{\pi}{2},\frac{\pi}{2}]$. Using the above conclusion, for $w_1'$, we can find some ${r_1}$ such that $\phi(w_1', e_{r_1})\le\frac{3}{2}\theta$. For $w_2'$, we can find different $i_1,i_2$ such that $\phi(w_2',e_{i_1})\le \frac{3}{2}\theta<(\frac{3}{2}+1)\theta$ and $\phi(w_2',e_{i_2})\le(\frac{3}{2}+1)\theta$. So we can find $r_2\neq r_1$ such that $\phi(w_2', e_{r_2})\le (\frac{3}{2}+1)\theta$. Following this scheme, we can find $r_j\notin\{r_1,\cdots, r_{j-1}\}$ and $\phi(w_j', e_{r_j})\le (j+\frac{1}{2})\theta<3m\theta$ for $j=1,\cdots,m$, as we have assumed that $m\le2(k+1)$. Let $w_j = \|w_j'\|_2e_{r_j}$, then we have constructed $(w_j)_{j=1}^m$ such that
\begin{align}
&\forall j\in\{1,\cdots,m\},\phi(w_j',w_j)\le 3m\theta\\
&\forall j\in\{1,\cdots,m\}, \|w_j'\|_2 = \|w_j\|_2\\
&\forall j\neq k, \rho(w_j,w_k)\ge \theta
\end{align}
Note that we have assumed that $\forall j=1,\cdots,m$, $\beta_j\in[-\frac{\pi}{2},\frac{\pi}{2}]$. In order to show that the conclusion holds for general $w_j'$, we need to consider the case where $\beta_j\in[-\frac{3}{2}\pi,-\frac{\pi}{2}]$. For that case, we can let $\beta_j'=\beta_j+\pi$, then $\beta_j'\in[-\frac{\pi}{2},\frac{\pi}{2}]$. Let $w_j''=\|w_j'\|_2(\sin\beta_j',\cos\beta_j')$, we can find the $e_{r_j}$ such that $\phi(w_j'',e_{r_j})\le m\theta$ following the same procedure. Let $w_j=-\|w_j'\|_2e_{r_j}$, then $\phi(w_j',w_j) = \phi(w_j'',e_{r_j})\le 2m\theta$ and as $\rho(-e_{r_j},e_k)=\rho(e_{r_j},e_k)$, the $\rho(w_j,w_k)\ge \theta$ condition is still satisfied. Also note that $\phi(a,b)\le\pi$, the proof for 2-dimensional case is completed.

Now we consider a general $d$-dimensional case. Similar to the 2-dimensional one, we construct a set of vectors with unit $l_2$ norm such that the pairwise angles $\rho(w_j,w_k)\ge \theta$ for $j\neq k$. We do the construction in two phases:

In the first phase, we construct a sequence of unit vector sets indexed by $\mathcal{I} = \{-(k+1),\cdots,-1,1,\cdots,k+1\}$: 
\begin{equation}
\forall i \in \mathcal{I}, \mathcal{E}_i = \{e\in\mathbb{R}^d|\|e\|_2=1, e\cdot(1,0,\cdots,0)=\cos\theta_i\}
\end{equation}
where $\theta_i = \mathrm{sgn}(i)(\frac{\theta}{2} + (|i|-1)\theta)$ is defined the same as we did in Eq.(\ref{eq:def_theta}). It can be shown that $\forall i\neq j$, $\forall e_i \in \mathcal{E}_i, e_j \in \mathcal{E}_j$,
\begin{equation}
\rho(e_i,e_j)\ge \theta
\end{equation}
The proof is as follows.
First, we write $e_i$ as $e_i = (\cos\theta_i,0,\cdots,0) + r_i$, where $\|r_i\|_2 = |\sin \theta_i|$. Similarly, $e_j = (\cos\theta_j,0,\cdots,0) + r_j$, where $\|r_j\|_2 = |\sin \theta_j|$. Hence we have
\begin{equation}
e_i \cdot e_j = \cos\theta_i \cos\theta_j + r_i \cdot r_j
\end{equation}
Hence 
\begin{equation}
\begin{array}{lll}
&&\cos(\rho(e_i,e_j)) \\
&=&|e_i\cdot e_j|\\
&\le& \cos\theta_i\cos\theta_j+|\sin\theta_i \sin\theta_j|\\
&=& \max(\cos(\theta_i+\theta_j), \cos(\theta_i-\theta_j))
\end{array}
\end{equation}
We have shown in the 2-dimensional case that $\cos(\theta_i+\theta_j)\ge\cos\theta$ and $\cos(\theta_i-\theta_j)\ge\cos\theta$, hence $\rho(e_i,e_j)\ge \theta$. In other words, we have proved that for any two vectors from $\mathcal{E}_i$ and $\mathcal{E}_j$, their pairwise angle is lower bounded by $\theta$. Now we proceed to construct a set of vectors for each $\mathcal{E}_i$ such that the pairwise angles can also be lower bounded by $\theta$. The construction is as follows.

First, we claim that for any $\mathcal{E}_i$, if $\mathcal{W}\subset \mathcal{E}$ satisfies
\begin{equation}
\forall w_j\neq w_k \in \mathcal{W}, \phi(w_j,w_k)\ge \theta
\end{equation}
then $|W|$ is finite. In order to prove that, we first define $B(x,r) = \{y\in\mathbb{R}^n: \|y-x\|_2< r\}$. Then $\mathcal{E}_i\subset \cup_{e\in\mathcal{E}_i} B(e, \frac{1-\cos\frac{\theta}{2}}{1+\cos\frac{\theta}{2}})$. From the definition of $\mathcal{E}_i$, it is a compact set, so the open cover has a finite subcover. Therefore we have $\exists V\subset\mathcal{E}_i$ with $|V|$ being finite and
\begin{equation}
\mathcal{E}_i\subset \cup_{v\in V} B(v, \frac{1-\cos\frac{\theta}{2}}{1+\cos\frac{\theta}{2}})
\end{equation}
Furthermore, we can verify that $\forall v \in V, \forall e_1,e_2\in B(v, \frac{1-\cos\frac{\theta}{2}}{1+\cos\frac{\theta}{2}})$, $\phi(e_1,e_2)\le\theta$. So if $\mathcal{W}\subset\mathcal{E}_i$ satisfies $\forall w_j\neq w_k \in \mathcal{W}, \phi(w_j,w_k)\ge \theta$, then for each $v$, $|B(v, \frac{1-\cos\frac{\theta}{2}}{1+\cos\frac{\theta}{2}})\cap\mathcal{W}|=1$. As $\mathcal{W}\subset\mathcal{E}_i$, we have
\begin{equation}
\begin{array}{lll}
&&|W| \\
&=& |W\cap \mathcal{E}_i|\\
&=& |W \cap (\cup_{v\in V} B(v, \frac{1-\cos\frac{\theta}{2}}{1+\cos\frac{\theta}{2}}))|\\
&=& |\cup_{v\in V}W\cap B(v, \frac{1-\cos\frac{\theta}{2}}{1+\cos\frac{\theta}{2}})|\\
&\le& \sum_{v\in V} |W\cap B(v, \frac{1-\cos\frac{\theta}{2}}{1+\cos\frac{\theta}{2}})|\\
&\le& \sum_{v\in V} 1\\
&=& |V|
\end{array}
\end{equation}
Therefore, we have proved that $|W|$ is finite. Using that conclusion, we can construct a sequence of vectors $w_j\in\mathcal{E}_i (j=1,\cdots,l)$ in the following way:
\begin{enumerate}
\item Let $w_1\in\mathcal{E}_i$ be any vector in $\mathcal{E}_i$.
\item For $j=2,\cdots$, let $w_j\in\mathcal{E}_i$ be any vector satisfying
\begin{align}
\label{eq:construct_w}
&\forall k=1,\cdots,j-1, \phi(w_j,w_k)\ge\theta\\
&\exists k\in\{,\cdots,j-1\}, \phi(w_j,w_k)=\theta
\end{align}
until we cannot find such vectors any more.
\item As we have proved that $|W|$ is finite, the above process will end in finite steps. Assume that the last vector we found is indexed by $l$.
\end{enumerate}
We can verify that such constructed vectors satisfy
\begin{equation}
\forall j\neq k\in\{1,\cdots,l\}, \rho(w_j,w_k)\ge \theta
\end{equation}
Note that due to the construction, $\phi(w_j,w_k)\ge\theta$, as $\rho(w_j,w_k)=\min(\phi(w_j,w_k),\pi-\phi(w_j,w_k))$, we only need to show that $\pi-\phi(w_j,w_k)\ge\theta$. To show that, we use the definition of $\mathcal{E}_i$ to write $w_j$ as $w_j = (\cos\theta_i,0,\cdots,0) + r_j$, where $\|r_j\|_2 = |\sin \theta_i|$. Similarly, $w_k = (\cos\theta_i,0,\cdots,0) + r_k$, where $\|r_k\|_2 = |\sin \theta_i|$. Therefore $\cos(\phi(w_j,w_k))=w_j\cdot w_k \ge \cos^2\theta_i-\sin^2\theta_i=\cos (2\theta_i)\ge\cos(\pi-\theta)$, where the last inequality follows from the construction of $\theta_i$. So $\pi-\phi(w_j,w_k)\ge\theta$, the proof for $\rho(w_j,w_k)\ge\theta$ is completed.

Now we will show that $\forall e\in\mathcal{E}_i$, we can find $j\in\{1,\cdots,l\}$ such that $\phi(e,w_j)\le \theta$. We prove it by contradiction: assume that there exists $e$ such that $\min_{j\in\{1,\cdots,l\}} \phi(e,w_j)>\theta$, then as $\mathcal{E}_j$ is a connected set, there is a path $q:t\in[0,1]\to \mathcal{E}_j$ connecting $e$ to $w_1$, and when $t=0$, the path starts at $q(0)=e$; when $t=1$, the path ends at $q(1)=w_1$. We define functions $r_j(t) = \phi(q(t), w_j)$ for $t\in[0,1]$ and $j=1,\cdots,l$. It is straightforward to see that $r_j(t)$ is continuous, hence $\min_j(r_j(t))$ is also continuous. As $\min_j(r_j(0))>\theta$ and $\min_j(r_j(0))=0<\theta$, there exists $t^*\in(0,1)$ such that $\min_j(r_j(0))=\theta$. Then $q(t^*)$ satisfies Condition \ref{eq:construct_w}, which contradicts the construction in $W$ as the construction only ends when we cannot find such vectors. Hence we have proved that 
\begin{equation}
\forall e\in\mathcal{E}_i,\exists j\in\{1,\cdots,l\}, \phi(e,w_j)\le \theta
\end{equation}

Now we can proceed to prove the main lemma. For each $i\in\mathcal{I}$, we use Condition \ref{eq:construct_w} to construct a sequence of vectors $w_{ij}$. Then such constructed vectors $w_{ij}$ have pairwise angles greater than or equal to $\theta$. Then for any $e \in \mathcal{R}^d$ with $\|e\|_2=1$, we write $e$ in sphere coordinates as $e=(\cos r_1, \sin r_1\cos r_2 ,\cdots,\prod_{j=1}^d\sin r_j)$. Use the same method as we did for the 2-dimensional case, we can find $\theta_i$ such that $|\theta_i-r|\le\frac{3}{2}\theta$. Then $e'=(\cos\theta_i, \sin 
\theta_i \cos r_2,\cdots, \sin\theta_i \prod_{j=2}^d\sin r_j)\in \mathcal{E}_i$. It is easy to verify that $\phi(e, e')=|\theta_i-r|\leq\frac{3}{2}\theta$. As $e'\in\mathcal{E}_i$, there exists $w_{ij}$ as we constructed such that $\phi(e',w_{ij})\leq\theta$. So $\phi(e,w_{ij})\leq\phi(e,e')+\phi(e',w_{ij})\le \frac{5}{2}\theta<3\theta$. So we have proved that for any $e\in \mathbb{R}^d$ with $\|e\|_2=1$, we can find $w_{ij}$ such that $\phi(e,w_{ij}) < 3\theta$.

For any $w_{ij}$, assume $i+1\in \mathcal{I}$, we first project $w_{ij}$ onto $w^*\in\mathcal{E}_{i+1}$ with $\phi(w_{ij},w^*)\le\frac{3}{2}\theta$, then we find $w_{i+1,j'}\in\mathcal{E}_{i+1}$ such that $\phi(w_{i+1,j'},w^*)\le\theta$. So we have found $w_{i+1,j'}$ such that $\phi(w_{ij},w_{i+1,j'})\le\frac{5}{2}\theta < 3 \theta$. We can use similar scheme to prove that $\forall w_{ij}$, there exists different $w_{i_1,j_1}\cdots, w_{i_{2k+1},j_{2k+1}}$ such that $(i_r,j_r)\neq (i,j)$ and $\phi(w_{ij},w_{i_r,j_r})\le 3r\theta$. Following the same proof as the 2-dimensional case, we can prove that if $m\le 2k+1$, then we can find a set of vectors $(w_j)_{j=1}^m$ such that 
\begin{align}
&\forall j\in\{1,\cdots,m\},\phi(w_j',w_j)\le \min(3m\theta,\pi)\\
&\forall j\in\{1,\cdots,m\}, \|w_j'\|_2 = \|w_j\|_2\\
&\forall j\neq k, \rho(w_j,w_k)\ge \theta
\end{align}
The proof completes.
\end{proof}
\begin{lemma}
\label{lem:f_prime_f_bound}
For any $f' \in \mathcal{F}'$, $\exists f \in \mathcal{F}''$ such that
\begin{equation}
\|f' - f\|_L \le 4mCC_1C_3\sin(\frac{\theta'}{2})
\end{equation}
where $\theta' = \min(3m\theta,{\pi})$.
\end{lemma}
\begin{proof}
According to the definition of $\mathcal{F}'$, $\forall f' \in \mathcal{F}'$, there exists $(\alpha_j')_{j=1}^m$, $(w_j')_{j=1}^m$ such that
\begin{align}
&f' = \sum_{j=1}^m \alpha_j' h(w_j'^T x)\\
&\forall j\in\{1,\cdots,m\}, |\alpha_j'|\le 2 C\\
&\forall j\in\{1,\cdots,m\}, \|w_j'\|_2 \le C_4
\end{align}
According to Lemma \ref{lem:theta_appro}, there exists $(w_j)_{j=1}^m$ such that
\begin{align}
&\forall j\neq k\in\{1,\cdots,m\}, \rho(w_j,w_k)\ge \theta\\
&\forall j\in\{1,\cdots,m\}, \|w_j\|_2 = \|w_j'\|_2\\
&\forall j \in \{1,\cdots,m\},\arccos (\frac{w_j \cdot w_j'}{\|w_j\|_2\|w_j'\|_2}) \le \theta'
\end{align}
where $\theta' = \min(m\theta,\frac{\pi}{2})$. Let $f = \sum_{j=1}^m \alpha_j h(w_j'^T x)$, then $\|\alpha\|_2 \le \sqrt{\|\alpha\|_1\|\alpha\|_\infty} \le 2\sqrt{m}C\le C_4$. Hence $f\in\mathcal{F}$. Then all we need to do is to bound $\|f-f'\|_L$:
\begin{equation}
\label{eq:bound_f_f_prime}
\begin{array}{lll}
&&\|f-f'\|_L^2\\
&=& \int_{\|x\|_2\le C_1} (f(x)-f'(x))^2 dP(x)\\
&=& \int_{\|x\|_2\le C_1} (\sum_j \alpha_j h(w_j^T x)-\sum_j \alpha_j h(w_j'^T x))^2 dP(x)\\
&=& \int_{\|x\|_2\le C_1} (\sum_j \alpha_j (h(w_j^T x)- h(w_j'^T x)))^2 dP(x)\\
&\le& \int_{\|x\|_2\le C_1} (\sum_j |\alpha_j| |w_j^T x- w_j'^T x|)^2 dP(x)\\
&\le& C_1^2\int_{\|x\|_2\le C_1} (\sum_j |\alpha_j| \|w_j- w_j'\|_2)^2 dP(x)
\end{array}
\end{equation}
As $\arccos (\frac{w_j \cdot w_j'}{\|w_j\|_2\|w_j'\|_2}) \le \theta'$, we have $w_j\cdot w_j'\ge \|w_j\|_2^2 \cos \theta'$. Hence
\begin{equation}
\begin{array}{lll}
&&\|w_j- w_j'\|_2^2\\
&=& 2\|w_j\|_2^2 - 2 w_j \cdot w_j'\\
&\le& 2\|w_j\|_2^2 - 2\|w_j\|_2^2 \cos \theta'\\
&\le& 4C_3^2\sin^2(\frac{\theta'}{2})
\end{array}
\end{equation}
Substituting back to Eq.(\ref{eq:bound_f_f_prime}), we have
\begin{equation}
\begin{array}{lll}
&&\|f-f'\|_L^2\\
&\le& C_1^2\int_{\|x\|_2\le C_1} (\sum_j |\alpha_j| 2 C_3 \sin(\frac{\theta'}{2}) )^2 dP(x)\\
&\le& 16m^2C^2C_1^2C_3^2\sin^2(\frac{\theta'}{2})
\end{array}
\end{equation}
\end{proof}
With this lemma, we can proceed to prove Theorem \ref{thm:appro2}. For every $g\in \Gamma_C$ with $g(0)=0$, according to Lemma \ref{lem:appro2}, $\exists f' \in \mathcal{F'}$ such that
\begin{equation}
\|g-f'\|_L \le 2C(\frac{1}{\sqrt{n}} + \frac{1+ 2\ln C_1C_4}{C_1C_4})
\end{equation}
According to Lemma \ref{lem:f_prime_f_bound}, we can find $f\in \mathcal{F}$ such that 
\begin{equation}
\label{eq:appro2}
\|f-f'\|_L \le 4mCC_1C_3\sin(\frac{\theta'}{2})
\end{equation}
The proof is completed by noting
\begin{equation}
\|g-f\|_L \le \|g-f'\|_L+\|f'-f\|_L
\end{equation}

\begin{figure*}
\begin{center}
\includegraphics[width=0.5\columnwidth]{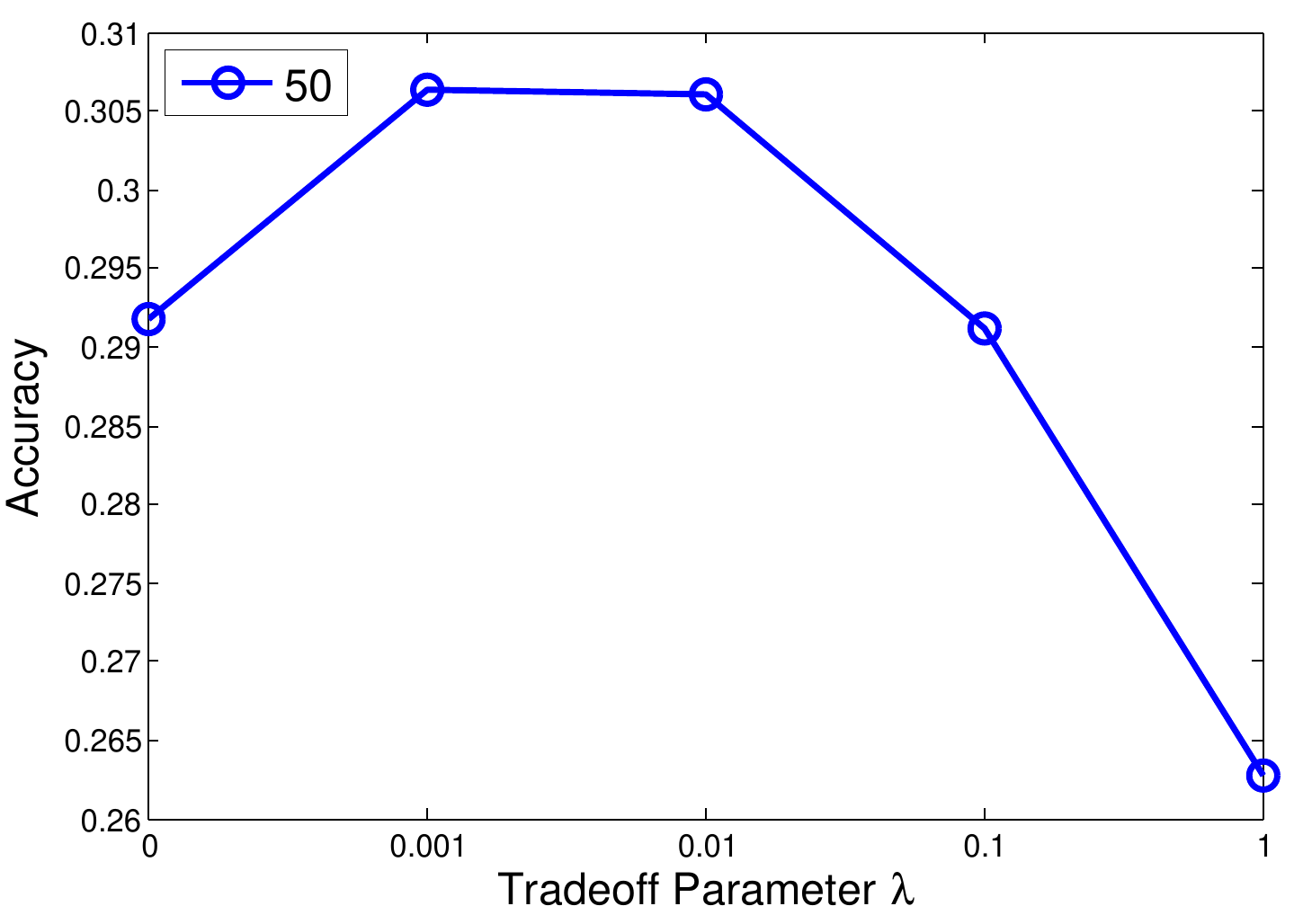}
\includegraphics[width=0.5\columnwidth]{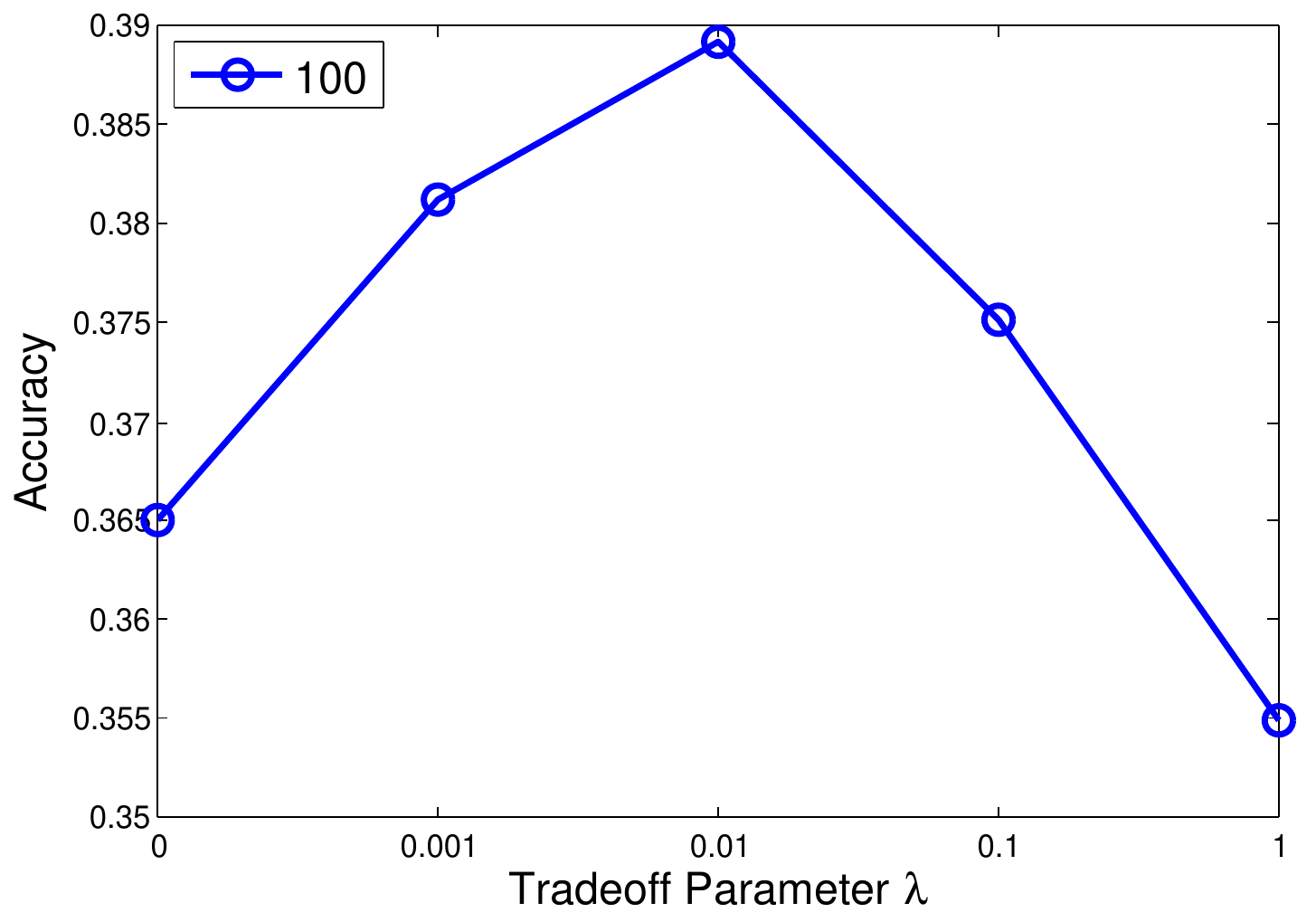}
\includegraphics[width=0.5\columnwidth]{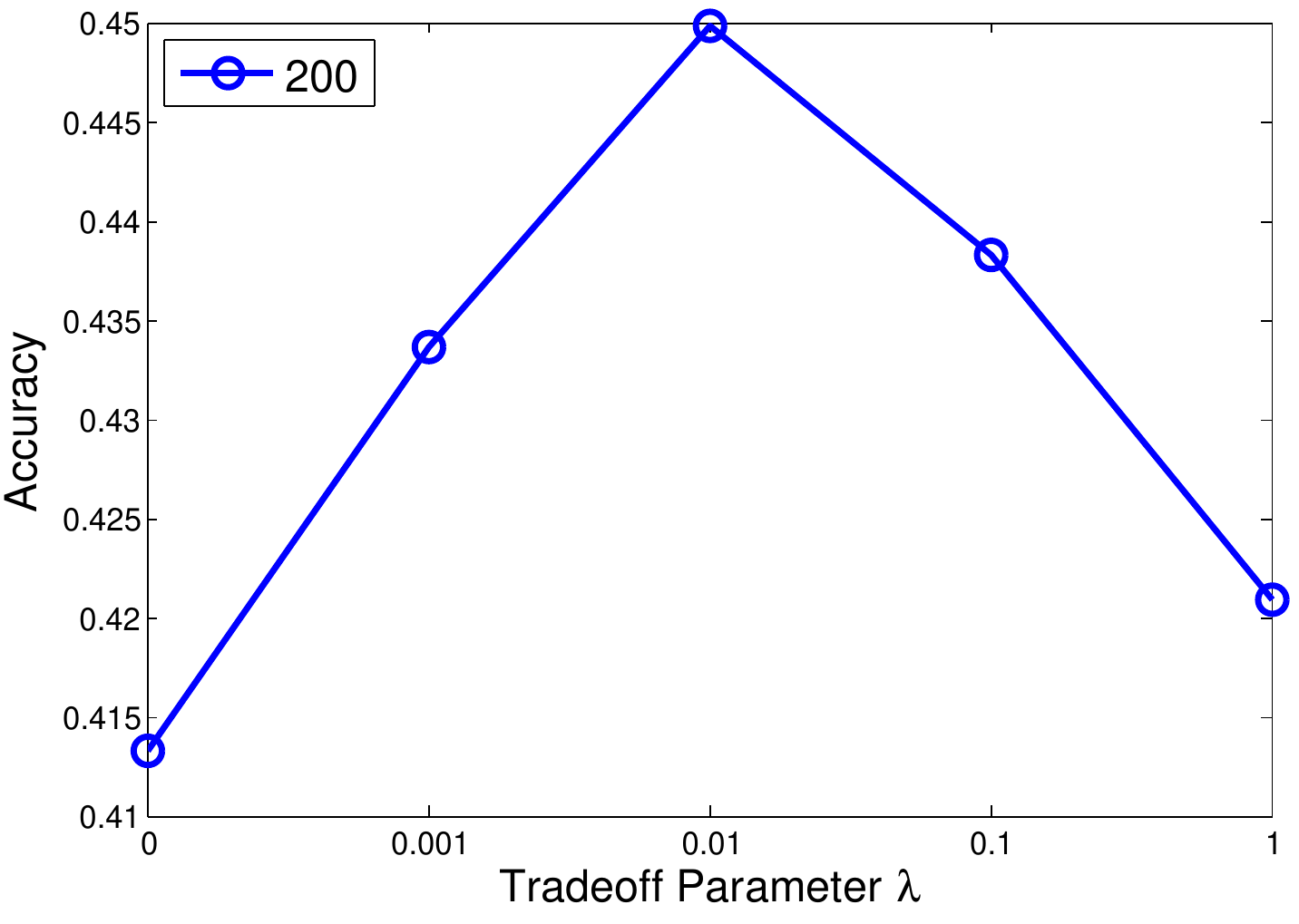}
\includegraphics[width=0.5\columnwidth]{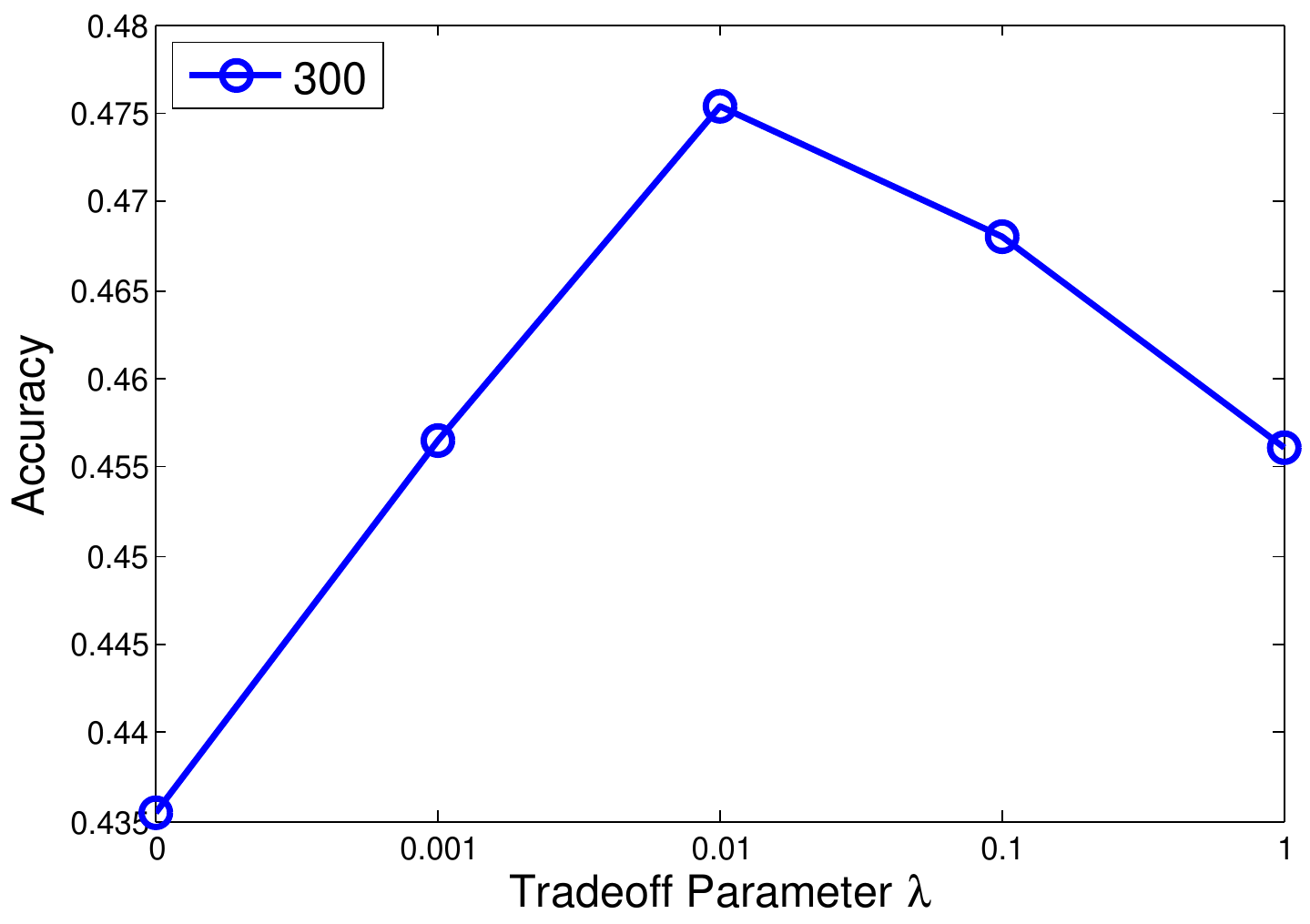}
\caption{Test accuracy versus $\lambda$ for neural networks with one hidden layer.}
\label{fig:acc_lambda_1l_timit}
\end{center}
\end{figure*}

\section{Experiments}
In this section, we present the experimental results on MAR-NN. Specifically, we are interested in how the performance of neural networks varies as the tradeoff parameter $\lambda$ in MAR-NN increases. A larger $\lambda$ would induce a stronger regularization, which generates a larger angle lower bound $\theta$. We apply MAR-NN for phoneme classification \cite{mohamed2011deep} on the TIMIT\footnote{\url{https://catalog.ldc.upenn.edu/LDC93S1}} speech dataset. 
The inputs are MFCC features extracted with context windows and the outputs are class labels generated by the HMM-GMM model through forced alignment \cite{mohamed2011deep}. The feature dimension is 360 and the number of classes is 2001. There are 1.1 million data instances in total. We use 70\% data for training and 30\% for testing. The activation function is sigmoid and loss function is cross-entropy. The networks are trained with stochastic gradient descent and the minibatch size is 100.

Figure \ref{fig:acc_lambda_1l_timit} shows the testing accuracy versus the tradeoff parameter $\lambda$ achieved by four neural networks with one hidden layer. The number of hidden units varies in $\{50,100,200,300\}$. 
As can be seen from these figures, under various network architectures, the best accuracy is achieved under a properly chosen $\lambda$. For example, for the neural network with 100 hidden units, the best accuracy is achieved  when $\lambda=0.01$. These empirical observations are aligned with our theoretical analysis that the best generalization performance is achieved under a proper diversity level. Adding this regularizer greatly improves the performance of neural networks, compared with unregularized NNs. For example, in a NN with 200 hidden units, the mutual angular regularizer improves the accuracy from $\sim$0.415 (without regularization) to 0.45.

\section{Related Works}
\subsection{Diversity-Promoting Regularization}
Diversity-promoting regularization approaches, which encourage the parameter vectors in machine learning models to be different from each other, have been widely studied and found many applications. Early works \cite{krogh1995neural,kuncheva2003measures,brown2005diversity,banfield2005ensemble,
tang2006analysis,partalas2008focused,yu2011diversity} explored how to select a diverse subset of base classifiers or regressors in ensemble learning, with the aim to improve generalization error and reduce computational complexity. Recently, \cite{Zou_priorsfor,xie2015diversifying,xie2015learning} studied the diversity regularization of latent variable models, with the goal to capture long-tail knowledge and reduce model complexity. In a multi-class classification problem, \cite{malkin2008ratio} proposed to use the determinant of the covariance matrix to encourage classifiers to be different from each other.
Our work focuses on the theoretical analysis of diversity regularized latent variable models, using neural network as an instance to study how the mutual angular regularizer affects the generalization error.
\subsection{Regularization of Neural Networks}
Among the vast amount of neural network research, a large body of works have been devoted to regularizing the parameter learning of NNs \cite{larsen1994generalization,hinton2012improving}, to restrict model complexity, prevent overfitting and achieve better generalization on unseen data. Widely studied and applied regularizers include L1 \cite{bach2014breaking}, L2 regularizers \cite{larsen1994generalization,bishop2006pattern}, early stopping \cite{bishop2006pattern}, dropout \cite{hinton2012improving} and DropConnect \cite{wan2013regularization}. In this paper, we study a new type of regularization approach of NN: diversity-promoting regularization, which bears new properties and functionalities complementary to the existing regularizers. 
%
\subsection{Generalization Performance of Neural Networks}
The generalization performance of neural networks, in particular the approximation error and estimation error, has been widely studied in the past several decades. For the approximation error, \cite{cybenko1989approximation} demonstrated that finite linear combinations of compositions
of a fixed, univariate function and a set of affine functionals can uniformly
approximate any continuous function.
\cite{hornik1991approximation} showed that neural networks with a single hidden layer, sufficiently many hidden units and arbitrary bounded and nonconstant activation function are universal approximators. \cite{leshno1993multilayer} proved that multilayer feedforward networks with a non-polynomial activation function can approximate any function.
 Various error rates have also been derived based on different assumptions of the target function. \cite{jones1992simple} showed that if the target function is in the hypothesis set formed by neural networks with one hidden layer of $m$ units, then the approximation error rate is $O(1/\sqrt{m})$.
 \cite{barron1993universal} showed that neural networks with one layer of $m$ hidden units and sigmoid activation function can achieve approximation error of order $O(1/\sqrt{m})$, where the target function is assumed to have a bound on the first moment of the magnitude distribution of the Fourier transform. \cite{makovoz1998uniform} proved that if the target function is of the form $f(\mb{x})=\int_{Q}c(\mb{w},b)h(\mb{w}^{\mathsf{T}}\mb{x}+b)\mathrm{d}\mu$, where $c(\cdot,\cdot)\in L_{\infty}(Q,\mu)$, then neural networks with one layer of $m$ hidden units can approximate it with an error rate of $n^{-1/2-1/(2d)}\sqrt{\log n}$, where $d$ is the dimension of input $\mb{x}$. 
 As for the estimation error, please refer to \cite{anthony1999neural} for an extensive review, which introduces various estimation error bounds based on VC-dimension, flat-shattering dimension, pseudo dimension and so on.

\section{Conclusions}
In this paper, we provide theoretical analysis regarding why the diversity-promoting regularizers can lead to better latent variable modeling. Using neural network as an instance, we analyze how the generalization performance of supervised latent variable models is affected by the mutual angular regularizer. Our analysis shows that increasing the diversity of hidden units leads to the decrease of estimation error bound and increase of approximation error bound. Overall, if the diversity level is set appropriately, a low generalization error can be achieved. The empirical experiments demonstrate that with mutual angular regularization, the performance of neural networks can be greatly improved and the empirical observations are consistent with the theoretical implications.

{\small
\bibliographystyle{unsrt}
\bibliography{drnn}

\begin{thebibliography}{10}

\bibitem{xie2015diversifying}
Pengtao Xie, Yuntian Deng, and Eric~P. Xing.
\newblock Diversifying restricted boltzmann machine for document modeling.
\newblock In {\em ACM SIGKDD Conference on Knowledge Discovery and Data
  Mining}, 2015.

\bibitem{bishop2006pattern}
Christopher~M Bishop.
\newblock {\em Pattern recognition and machine learning}.
\newblock springer, 2006.

\bibitem{han2011data}
Jiawei Han, Micheline Kamber, and Jian Pei.
\newblock {\em Data mining: concepts and techniques: concepts and techniques}.
\newblock Elsevier, 2011.

\bibitem{fukunaga2013introduction}
Keinosuke Fukunaga.
\newblock {\em Introduction to statistical pattern recognition}.
\newblock Academic press, 2013.

\bibitem{council2013frontiers}
N~Council.
\newblock Frontiers in massive data analysis, 2013.

\bibitem{jordan2015machine}
MI~Jordan and TM~Mitchell.
\newblock Machine learning: Trends, perspectives, and prospects.
\newblock {\em Science}, 349(6245):255--260, 2015.

\bibitem{rabiner1989tutorial}
Lawrence~R Rabiner.
\newblock A tutorial on hidden markov models and selected applications in
  speech recognition.
\newblock {\em Proceedings of the IEEE}, 77(2):257--286, 1989.

\bibitem{bishop1998latent}
Christopher~M Bishop.
\newblock Latent variable models.
\newblock In {\em Learning in graphical models}, pages 371--403. Springer,
  1998.

\bibitem{knott1999latent}
Martin Knott and David~J Bartholomew.
\newblock {\em Latent variable models and factor analysis}.
\newblock Number~7. Edward Arnold, 1999.

\bibitem{blei2003latent}
David~M Blei, Andrew~Y Ng, and Michael~I Jordan.
\newblock Latent dirichlet allocation.
\newblock {\em Journal of machine Learning research}, 2003.

\bibitem{hinton2006fast}
Geoffrey~E Hinton, Simon Osindero, and Yee-Whye Teh.
\newblock A fast learning algorithm for deep belief nets.
\newblock {\em Neural computation}, 18(7):1527--1554, 2006.

\bibitem{airoldi2009mixed}
Edoardo~M Airoldi, David~M Blei, Stephen~E Fienberg, and Eric~P Xing.
\newblock Mixed membership stochastic blockmodels.
\newblock In {\em Advances in Neural Information Processing Systems}, pages
  33--40, 2009.

\bibitem{blei2014build}
David~M Blei.
\newblock Build, compute, critique, repeat: Data analysis with latent variable
  models.
\newblock {\em Annual Review of Statistics and Its Application}, 2014.

\bibitem{rumelhart1985learning}
David~E Rumelhart, Geoffrey~E Hinton, and Ronald~J Williams.
\newblock Learning internal representations by error propagation.
\newblock Technical report, DTIC Document, 1985.

\bibitem{deerwester1990indexing}
Scott~C. Deerwester, Susan~T Dumais, Thomas~K. Landauer, George~W. Furnas, and
  Richard~A. Harshman.
\newblock Indexing by latent semantic analysis.
\newblock {\em Journal of the American Society for Information Science},
  41(6):391--407, 1990.

\bibitem{olshausen1997sparse}
Bruno~A Olshausen and David~J Field.
\newblock Sparse coding with an overcomplete basis set: A strategy employed by
  v1?
\newblock {\em Vision research}, 1997.

\bibitem{lee1999learning}
Daniel~D Lee and H~Sebastian Seung.
\newblock Learning the parts of objects by non-negative matrix factorization.
\newblock {\em Nature}, 401(6755):788--791, 1999.

\bibitem{xing2002distance}
Eric~P Xing, Michael~I Jordan, Stuart Russell, and Andrew~Y Ng.
\newblock Distance metric learning with application to clustering with
  side-information.
\newblock In {\em Advances in neural information processing systems}, pages
  505--512, 2002.

\bibitem{hinton2012deep}
Geoffrey Hinton, Li~Deng, Dong Yu, George~E Dahl, Abdel-rahman Mohamed, Navdeep
  Jaitly, Andrew Senior, Vincent Vanhoucke, Patrick Nguyen, Tara~N Sainath,
  et~al.
\newblock Deep neural networks for acoustic modeling in speech recognition: The
  shared views of four research groups.
\newblock {\em Signal Processing Magazine, IEEE}, 2012.

\bibitem{xing2007bayesian}
Eric~P Xing, Michael~I Jordan, and Roded Sharan.
\newblock Bayesian haplotype inference via the dirichlet process.
\newblock {\em Journal of Computational Biology}, 14(3):267--284, 2007.

\bibitem{song2009keller}
Le~Song, Mladen Kolar, and Eric~P Xing.
\newblock Keller: estimating time-varying interactions between genes.
\newblock {\em Bioinformatics}, 25(12):i128--i136, 2009.

\bibitem{gunawardana2008tied}
Asela Gunawardana and Christopher Meek.
\newblock Tied boltzmann machines for cold start recommendations.
\newblock In {\em Proceedings of the 2008 ACM conference on Recommender
  systems}, pages 19--26. ACM, 2008.

\bibitem{koren2009matrix}
Yehuda Koren, Robert Bell, and Chris Volinsky.
\newblock Matrix factorization techniques for recommender systems.
\newblock {\em IEEE Computer}, (8):30--37, 2009.

\bibitem{wang2014peacock}
Yi~Wang, Xuemin Zhao, Zhenlong Sun, Hao Yan, Lifeng Wang, Zhihui Jin, Liubin
  Wang, Yang Gao, Ching Law, and Jia Zeng.
\newblock Peacock: Learning long-tail topic features for industrial
  applications.
\newblock {\em ACM Transactions on Intelligent Systems and Technology}, 2014.

\bibitem{xie2015learning}
Pengtao Xie.
\newblock Learning compact and effective distance metrics with diversity
  regularization.
\newblock In {\em European Conference on Machine Learning}, 2015.

\bibitem{wang2015rubik}
Yichen Wang, Robert Chen, Joydeep Ghosh, Joshua~C Denny, Abel Kho, You Chen,
  Bradley~A Malin, and Jimeng Sun.
\newblock Rubik: Knowledge guided tensor factorization and completion for
  health data analytics.
\newblock In {\em Proceedings of the 21th ACM SIGKDD International Conference
  on Knowledge Discovery and Data Mining}, pages 1265--1274. ACM, 2015.

\bibitem{Zou_priorsfor}
James~Y. Zou and Ryan~P. Adams.
\newblock Priors for diversity in generative latent variable models.
\newblock In {\em Advances in Neural Information Processing Systems}, 2012.

\bibitem{hofmann1999probabilistic}
Thomas Hofmann.
\newblock Probabilistic latent semantic analysis.
\newblock In {\em Proceedings of the Fifteenth conference on Uncertainty in
  artificial intelligence}, pages 289--296. Morgan Kaufmann Publishers Inc.,
  1999.

\bibitem{krizhevsky2012imagenet}
Alex Krizhevsky, Ilya Sutskever, and Geoffrey~E Hinton.
\newblock Imagenet classification with deep convolutional neural networks.
\newblock In {\em Advances in neural information processing systems}, 2012.

\bibitem{bahdanau2014neural}
Dzmitry Bahdanau, Kyunghyun Cho, and Yoshua Bengio.
\newblock Neural machine translation by jointly learning to align and
  translate.
\newblock {\em arXiv preprint arXiv:1409.0473}, 2014.

\bibitem{wasserman2013all}
Larry Wasserman.
\newblock {\em All of statistics: a concise course in statistical inference}.
\newblock Springer Science \& Business Media, 2013.

\bibitem{anthony1999neural}
Martin Anthony and Peter~L Bartlett.
\newblock {\em Neural network learning: Theoretical foundations}.
\newblock cambridge university press, 1999.

\bibitem{bartlett2003rademacher}
Peter~L Bartlett and Shahar Mendelson.
\newblock Rademacher and gaussian complexities: Risk bounds and structural
  results.
\newblock {\em Journal of Machine Learning Research}, 3:463--482, 2003.

\bibitem{liang2015lecture}
Percy Liang.
\newblock Lecture notes of statistical learning theory.
\newblock 2015.

\bibitem{barron1993universal}
Andrew~R Barron.
\newblock Universal approximation bounds for superpositions of a sigmoidal
  function.
\newblock {\em Information Theory, IEEE Transactions on}, 1993.

\bibitem{mohamed2011deep}
Abdel-rahman Mohamed, Tara~N Sainath, George Dahl, Bhuvana Ramabhadran,
  Geoffrey~E Hinton, Michael Picheny, et~al.
\newblock Deep belief networks using discriminative features for phone
  recognition.
\newblock In {\em Acoustics, Speech and Signal Processing (ICASSP), 2011 IEEE
  International Conference on}, pages 5060--5063. IEEE, 2011.

\bibitem{krogh1995neural}
Anders Krogh, Jesper Vedelsby, et~al.
\newblock Neural network ensembles, cross validation, and active learning.
\newblock {\em Advances in neural information processing systems}, 1995.

\bibitem{kuncheva2003measures}
Ludmila~I Kuncheva and Christopher~J Whitaker.
\newblock Measures of diversity in classifier ensembles and their relationship
  with the ensemble accuracy.
\newblock {\em Machine learning}, 2003.

\bibitem{brown2005diversity}
Gavin Brown, Jeremy Wyatt, Rachel Harris, and Xin Yao.
\newblock Diversity creation methods: a survey and categorisation.
\newblock {\em Information Fusion}, 2005.

\bibitem{banfield2005ensemble}
Robert~E Banfield, Lawrence~O Hall, Kevin~W Bowyer, and W~Philip Kegelmeyer.
\newblock Ensemble diversity measures and their application to thinning.
\newblock {\em Information Fusion}, 2005.

\bibitem{tang2006analysis}
E~Ke Tang, Ponnuthurai~N Suganthan, and Xin Yao.
\newblock An analysis of diversity measures.
\newblock {\em Machine Learning}, 2006.

\bibitem{partalas2008focused}
Ioannis Partalas, Grigorios Tsoumakas, and Ioannis~P Vlahavas.
\newblock Focused ensemble selection: A diversity-based method for greedy
  ensemble selection.
\newblock In {\em European Conference on Artificial Intelligence}, 2008.

\bibitem{yu2011diversity}
Yang Yu, Yu-Feng Li, and Zhi-Hua Zhou.
\newblock Diversity regularized machine.
\newblock In {\em IJCAI Proceedings-International Joint Conference on
  Artificial Intelligence}. Citeseer, 2011.

\bibitem{malkin2008ratio}
Jonathan Malkin and Jeff Bilmes.
\newblock Ratio semi-definite classifiers.
\newblock In {\em Acoustics, Speech and Signal Processing, 2008. ICASSP 2008.
  IEEE International Conference on}, pages 4113--4116. IEEE, 2008.

\bibitem{larsen1994generalization}
Jan Larsen and Lars~Kai Hansen.
\newblock Generalization performance of regularized neural network models.
\newblock In {\em Neural Networks for Signal Processing [1994] IV. Proceedings
  of the 1994 IEEE Workshop}, 1994.

\bibitem{hinton2012improving}
Geoffrey~E Hinton, Nitish Srivastava, Alex Krizhevsky, Ilya Sutskever, and
  Ruslan~R Salakhutdinov.
\newblock Improving neural networks by preventing co-adaptation of feature
  detectors.
\newblock {\em arXiv preprint arXiv:1207.0580}, 2012.

\bibitem{bach2014breaking}
Francis Bach.
\newblock Breaking the curse of dimensionality with convex neural networks.
\newblock {\em arXiv preprint arXiv:1412.8690}, 2014.

\bibitem{wan2013regularization}
Li~Wan, Matthew Zeiler, Sixin Zhang, Yann~L Cun, and Rob Fergus.
\newblock Regularization of neural networks using dropconnect.
\newblock In {\em Proceedings of the 30th International Conference on Machine
  Learning}, 2013.

\bibitem{cybenko1989approximation}
George Cybenko.
\newblock Approximation by superpositions of a sigmoidal function.
\newblock {\em Mathematics of control, signals and systems}, 1989.

\bibitem{hornik1991approximation}
Kurt Hornik.
\newblock Approximation capabilities of multilayer feedforward networks.
\newblock {\em Neural networks}, 1991.

\bibitem{leshno1993multilayer}
Moshe Leshno, Vladimir~Ya Lin, Allan Pinkus, and Shimon Schocken.
\newblock Multilayer feedforward networks with a nonpolynomial activation
  function can approximate any function.
\newblock {\em Neural networks}, 1993.

\bibitem{jones1992simple}
Lee~K Jones.
\newblock A simple lemma on greedy approximation in hilbert space and
  convergence rates for projection pursuit regression and neural network
  training.
\newblock {\em The annals of Statistics}, 1992.

\bibitem{makovoz1998uniform}
Y~Makovoz.
\newblock Uniform approximation by neural networks.
\newblock {\em Journal of Approximation Theory}, 1998.

\end{thebibliography}
}

\end{document}